\documentclass{article}

\usepackage{microtype}

\usepackage{graphicx} % more modern
\usepackage{subcaption}

\usepackage{algorithm}
\usepackage{algorithmic}
\usepackage[export]{adjustbox}
\usepackage{pgf}
\usepackage{tikz}
\usetikzlibrary{positioning, arrows,automata}
\usepackage{enumitem}
\usepackage{mathtools}
\usepackage{dblfloatfix}
\usepackage{xcolor}
\usepackage{booktabs}
\usepackage{multirow}
\usepackage{tabularx}
\usepackage{makecell}
\usepackage{amsthm}
\usepackage{centernot}
\usepackage{amsmath}
\usepackage{amssymb}
\usepackage{fancyvrb}

\usepackage{marginnote}

\makeatletter
\let\oldmarginnote\marginnote
\renewcommand*{\marginnote}[1]{%
   \begingroup%
   \ifodd\value{page}
     \if@firstcolumn\reversemarginpar\fi
   \else
     \if@firstcolumn\else\reversemarginpar\fi
   \fi
   \oldmarginnote{#1}%
   \endgroup%
}
\makeatother

\usepackage{hyperref}

\usepackage[accepted]{icml2018}

\icmltitlerunning{Will it Blend? Composing Value Functions in Reinforcement Learning}

\DeclareMathOperator*{\argmax}{argmax}

\newtheorem{theorem}{Theorem}

\newtheorem{corollary}{Corollary}
\newtheorem{lemma}{Lemma}

\theoremstyle{definition}
\newtheorem{definition}{Definition}

\DeclarePairedDelimiterX{\setarg}[1]{\{}{\}}{%
  \ifnum\currentgrouptype=16 \else\begingroup\fi
  \activatebar#1
  \ifnum\currentgrouptype=16 \else\endgroup\fi
}

\DeclarePairedDelimiterX{\expectarg}[1]{[}{]}{%
  \ifnum\currentgrouptype=16 \else\begingroup\fi
  \activatebar#1
  \ifnum\currentgrouptype=16 \else\endgroup\fi
}

\DeclarePairedDelimiterX{\probarg}[1]{(}{)}{%
  \ifnum\currentgrouptype=16 \else\begingroup\fi
  \activatebar#1
  \ifnum\currentgrouptype=16 \else\endgroup\fi
}

\newcommand{\innermid}{\nonscript\;\delimsize\vert\nonscript\;}
\newcommand{\activatebar}{%
  \begingroup\lccode`\~=`\|
  \lowercase{\endgroup\let~}\innermid 
  \mathcode`|=\string"8000
}

\begin{document} 

\twocolumn[
\icmltitle{Will it Blend? \\ Composing Value Functions in Reinforcement Learning}



% you can specify symbols, otherwise they are numbered in order
\icmlsetsymbol{equal}{*}

\begin{icmlauthorlist}
\icmlauthor{Benjamin van Niekerk}{equal,wits}
\icmlauthor{Steven James}{equal,wits}
\icmlauthor{Adam Earle}{wits}
\icmlauthor{Benjamin Rosman}{wits,csir}

\end{icmlauthorlist}

\icmlaffiliation{wits}{University of the Witwatersrand, Johannesburg, South Africa}
\icmlaffiliation{csir}{Council for Scientific and Industrial Research, Pretoria, South Africa}

\icmlcorrespondingauthor{Benjamin van Niekerk}{benjamin.vanniekerk@students.wits.ac.za}
%\icmlcorrespondingauthor{George Konidaris}{gdk@cs.brown.edu}

\icmlkeywords{reinforcement learning, life long learning, composition, entropy-regularised reinforcement learning}

\vskip 0.3in
]



\printAffiliationsAndNotice{} % otherwise use the standard text.

\begin{abstract} 

An important property for lifelong-learning agents is the ability to combine existing skills to solve unseen tasks.
In general, however, it is unclear how to compose skills in a principled way. We provide a ``recipe'' for optimal value function composition in entropy-regularised reinforcement learning (RL) and then extend this to the standard RL setting.
Composition is demonstrated in a video game environment, where an agent with an existing library of policies is able to solve new tasks without the need for further learning.

\end{abstract} 

\section{Introduction}

A major challenge in artificial intelligence is creating agents capable of leveraging existing knowledge for inductive transfer.
Lifelong learning, in particular, requires that an agent be able to act effectively when presented with a new, unseen task.
A promising approach is to combine behaviours learned in various separate tasks to create new skills \cite{taylor09}.
This compositional approach allows us to build rich behaviours from relatively simple ones, resulting in a (good!) combinatorial explosion in the agent's abilities \citep{saxe17}.
However, in general, it is unclear how to produce new optimal skills from known ones.

One approach to compositionality is Linearly-solvable Markov Decision Processes (LMDPs) \citep{todorov07}, which structure the reward function to ensure that the Bellman equation becomes linear in the exponentiated value function. 
\citet{todorov09} proves that the optimal value functions of a set of LMDPs can be composed to produce the optimal value function for a composite task.
This is a particularly attractive property, since solving new tasks requires no further learning.
However, the LMDP framework has  so far been restricted to the tabular case with known dynamics, limiting its usefulness.  

Related work has focused on entropy-regularised reinforcement learning (RL) \citep{schulman17, haarnoja17,nachum17}, where rewards are augmented with an entropy-based penalty term. 
This has been shown to lead to improved exploration and rich, multimodal value functions.

Prior work \citep{haarnoja18} has demonstrated that these value functions can be composed to \textit{approximately} solve the intersection of tasks.
We complement these results by proving \textit{optimal} composition for the union of tasks in the total-reward, absorbing-state setting.
Thus, any task lying in the ``span'' of a set of basis tasks can be solved immediately, without any further learning.
We provide a ``recipe'' for optimally composing value functions, and demonstrate our method in a video game. 
Results show that an agent is able to compose existing policies learned from pixel input to generate new, optimal behaviours.

\section{Background}

A Markov decision process (MDP) is defined by the $4$-tuple $(\mathcal{S}, \mathcal{A}, \rho, r)$ where 
(i) the \textit{state space} $\mathcal{S}$ is standard Borel; 
(ii) the \textit{action space} $\mathcal{A}$ is finite (and therefore a compact metric space when equipped with the discrete metric); 
(iii) the transition dynamics $\rho$ define  a Markov kernel $(s, a) \mapsto \rho_{(s, a)}$ from $\mathcal{S} \times \mathcal{A}$ to $\mathcal{S}$; and 
(iv) the reward $r$ is a real-valued function on $\mathcal{S} \times \mathcal{A}$ that is bounded and measurable.

In RL, an agent's goal is to maximise its utility by making a sequence of decisions. 
At each time step, the agent receives an observation from $\mathcal{S}$ and executes an action from $\mathcal{A}$ according to its \textit{policy}. 
As a consequence of its action, the agent receives feedback (reward) and transitions to a new state.
Whereas the rewards represent only the immediate outcome, the utility captures the long-term consequences of actions.
Historically, many utility functions have been investigated \citep{puterman14}, but in this paper we only consider the total-reward criterion (see Section~\ref{sec:entropy_regularised}).

We consider the class of MDPs with an absorbing set $\mathcal{G}$, which is a Borel subset of the state space. We augment the state space with a virtual state $g$ such that $\rho_{(s, a)}(\{g\}) = 1$ for all $(a, s)$ in $\mathcal{G} \times \mathcal{A}$, and $r = 0$ after reaching $g$.
In the control literature, this class of MDPs is often called stochastic shortest path problems  \citep{bertsekas91}, and naturally model domains that terminate after the agent achieves some goal.

We restrict our attention to stationary Markov policies, or simply policies. 
%It has been shown (CITE) that the space of all policies in Borel dynamic programming is dominated by Markov policies, and as such there is reason to believe that similar results will hold in the entropy-regularised setting.
A policy $s \mapsto \pi_s$ is a Markov kernel from $\mathcal{S}$ to $\mathcal{A}$.
Together with an initial distribution $\nu$ over $\mathcal{S}$, a policy defines a probability measure over trajectories.
To formalise this, we construct the set of $n$-step histories inductively by defining $\mathcal{H}_0 = \mathcal{S}$ and $\mathcal{H}_{n} = \mathcal{H}_{n-1} \times \mathcal{A} \times \mathcal{S}$ for $n$ in $\mathbb{N}$. 
The $n$-step histories represent the set of all possible trajectories of length $n$ in the MDP. 
%Then, equipped with the product topology, $H_n$ is a standard Borel space endowed with the product $\sigma$-algebra, and in particular, $H_\infty$ is also a standard Borel space.
The probability measure on $\mathcal{H}_n$ induced by the policy $\pi$ is then
\[
P^\pi_{\nu,n} = \nu \otimes \underbrace{\pi \otimes \rho \otimes \cdots \otimes \pi \otimes \rho}_{n \text{ times}}.
\]
Using the standard construction \citep{klenke95}, we can define a unique probability measure $P^\pi_\nu$ on $\mathcal{H}_\infty$ consistent with the measures $P^\pi_{\nu, n}$ in the sense that
\[
P^\pi_\nu(\mathcal{E} \times \mathcal{A} \times \mathcal{S} \times \mathcal{A} \times \cdots) = P^\pi_{\nu, n}(\mathcal{E}),
\]
for any $n$ in $\mathbb{N}$ and any Borel set $\mathcal{E} \subseteq \mathcal{H}_n$.
If $\nu$ is concentrated on a single state $s$, we simply write $P^\pi_\nu = P^\pi_s$. 
Additionally for any real-valued bounded measurable function $f$ on $\mathcal{H}_n$, we define $\mathbb{E}^\pi_\nu[f]$ to be the expected value of $f$ under $P^\pi_\nu$.  

Finally, we introduce the notion of a {\emph{proper policy}}---a policy under which the probability of reaching $\mathcal{G}$ after $n$ steps converges to $1$ uniformly over $\mathcal{S}$ as $n \to \infty$.
Our definition extends that of \citet{bertsekas95} to general state spaces, and is equivalent to the definition of transient policies used by \citet{james06}:
\begin{definition}
A stationary Markov policy $\pi$ is said to be \textit{proper} if
\[
\sup_{s \in \mathcal{S}} \sum_{t = 0}^\infty P^\pi_s(s_t \not \in \mathcal{G}) < \infty.
\]
Otherwise, we say that $\pi$ is improper.
\end{definition}
%The class of proper policies is important in the analysis of absorbing-state MDPs. 
%In fact, we show in Section~\ref{sec:value_function} that the optimal policy is a proper policy under mild assumptions.

\subsection{Entropy-Regularised RL} \label{sec:entropy_regularised}

In the standard RL setting, the expected reward at state $s$ under policy $\pi$ is given by $\mathbb{E}_{a\sim\pi}\left[r(s, a)\right]$.  
Entropy-regularised RL \citep{ziebert10,fox15,haarnoja17,schulman17,nachum17} augments the reward function with a term that penalises deviating from some reference policy $\bar{\pi}$.
That is, the expected reward is given by  $\mathbb{E}_{a\sim\pi}\left[r(s, a)\right] - \tau \text{KL}[\pi_s||\bar{\pi}_s]$, where $\tau$ is a positive scalar temperature parameter and $\text{KL}[\pi_s||\bar{\pi}_s]$ is the Kullback-Leibler divergence between $\pi$ and the reference policy $\bar{\pi}$ at state $s$.
When $\bar{\pi}$ is the uniform random policy, the regularised reward is equivalent to the standard entropy bonus up to an additive constant \citep{schulman17}. 
This results in policies that are more robust to ``winner's curse'' \citep{fox15}.
Additionally, the reference policy can be used to encode prior knowledge through expert demonstration.  

%The effect of the regularisation term may be interpreted in a number of ways. When the problem occurs in the real-world, and the reference policy is taken to be the naturally evolving dynamics without control, then the regularization term can be thought of as specifying a bias for energy efficient controls. This interpretation holds more generally when the reference policy is abstractly thought of as biasing the policy search space.

Based on the above regularisation, we define the $n$-step value function starting from $s$ and following policy $\pi$ as:
\[
V_{\pi, n}(s) = \mathbb{E}^\pi_{s, n} \left[ \sum^{n-1}_{t = 0} r(s_t, a_t) - \tau \text{KL}[\pi_{s_t}||\bar{\pi}_{s_t}] \right].
\]
Note that since the KL-divergence term is measurable \citep[Lemma~1.4.3]{dupuis11}, $V_{\pi, n}$ is well-defined. The infinite-horizon value function, which represents the total expected return after executing $\pi$ from $s$, is then
\[
V_{\pi}(s) = \limsup_{n \to \infty} V_{\pi, n}(s).
\]
Since the reward function and KL-divergence are bounded,\footnote{Under the assumptions that $\mathcal{A}$ is finite and $\bar{\pi}$ is chosen so that $\pi_s$ is absolutely continuous with respect to $\bar{\pi}_s$ for any state $s$ and policy $\pi$.} $V_{\pi}$ is well defined.
Similarly, we define the $Q$-function to be the expected reward after taking action $a$ in state $s$, and thereafter following policy $\pi$:  
\begin{equation} \label{eq:q-function}
Q_\pi(s,a) = r(s,a) + \int_{\mathcal{S}} V_\pi(s^\prime) \rho_{(a,s)}(ds^\prime).
\end{equation} 

Given the definitions above, we say that a measurable function $V^*$ is optimal if $V^*(s) = \sup_\pi V_\pi(s)$ for all $s$ in $\mathcal{S}$.
Furthermore, a policy $\pi^*$ is optimal if $V_{\pi^*} = V^*$.

In the standard RL case, the optimal policy is always deterministic and is defined by $\argmax_a Q^*(s, a)$.
On the other hand, entropy-regularised problems may not admit an optimal deterministic policy.
This results from the KL-divergence term, which penalises deviation from the reference policy $\bar{\pi}$.
If $\bar{\pi}$ is stochastic, then a deterministic policy may incur more cost than a stochastic policy.
To see this, consider the simple two-state MDP shown in Figure~\ref{fig:counter}:

\begin{figure}[h]
\label{fig:counter-example}
\centering
\begin{tikzpicture}[shorten >=1pt,node distance=4cm,on grid,thick]
  \node[state] (s) {$s$};
  \node[state, accepting] (g)[right=of s] {$g$};
  \path[->] (s) edge node [above] {$r(s, \texttt{Right}) = -1$} (g) 
                   edge [loop left] node [left] {$r(s, \texttt{Left}) = -1$} ();                
\end{tikzpicture}
\caption{A two-state MDP with absorbing state $g$. } \label{fig:counter}
\end{figure}
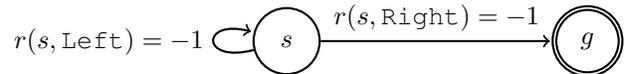

Given $\tau > 0$ and a uniformly random reference policy, let $\pi$ be the deterministic policy that selects \texttt{Right} with probability $1$, and let $\pi_\varepsilon$ be the stochastic policy that selects \texttt{Right} with probability $1 - \varepsilon$ and \texttt{Left} with probability $\varepsilon$.
Then, choosing $\varepsilon$ small enough, we can guarantee that $V_{\pi_\varepsilon} > V_\pi$.
Therefore for any $\tau > 0$, the optimal policy is non-deterministic.

\begin{proof}
First, the value of state $s$ under the policy $\pi$ is given by $V_\pi(s) = -1 - \tau \log 2$. On the other hand, the expected number of steps from $s$ to $g$ under $\pi_\varepsilon$ is $1/(1- \varepsilon)$ so we have
\[
V_{\pi_\varepsilon}(s) = -\frac{1 + \varepsilon\log 2\varepsilon }{1 - \varepsilon} - \log 2(1 - \varepsilon).
\]
Now, choose $\varepsilon$ such that
\begin{equation}
\label{eq:counter-example-epsilon}
\varepsilon < \frac{\tau(\log2 - 1/2)}{2 + \tau}.
\end{equation}
Then, from \eqref{eq:counter-example-epsilon} we have that:
\begin{enumerate}[nolistsep, label=(\roman*)]
\item $\varepsilon < 1/2$ and therefore $\log 2(1-\varepsilon) < 0$,
\item $\log 2\varepsilon < 2 \varepsilon$, and
\item $\log 2 - 1/2 < 1$ and therefore $\varepsilon < \tau / (2 + \tau)$
\end{enumerate}

Using the above facts we get the chain of inequalities:
\begin{align*}
V_{\pi_\varepsilon}(s) &\overset{\text{(i)}} {>} -\frac{1 + \varepsilon\log(2\varepsilon)}{1 - \varepsilon} \overset{\text{(ii)}}{>} -\frac{1 + 2\varepsilon}{1 - \varepsilon} \\
&\overset{\text{(iii)}}{>} -1 - \tau/2 - \varepsilon(2 + \tau) \\ &\overset{\text{(2)}}{>} -1 - \tau \log 2.
\end{align*}
The last inequality follows directly from \eqref{eq:counter-example-epsilon}, giving $V_{\pi_\varepsilon}(s) > V_\pi(s)$.
\end{proof}

\section{Soft Value and Policy Iteration} \label{sec:value_function}

In this section, we investigate the total-reward, entropy-regularised criterion defined above.
While value and policy iteration in entropy-regularised RL have been analysed previously \citep{nachum17}, convergence results are limited to discounted MDPs.
We sketch an argument that an optimal proper policy exists under the total-reward criterion and that the soft versions of value and policy iteration (see Algorithms \ref{alg:value-iteration} and \ref{alg:policy-iteration}) converge to optimal solutions.

We begin by defining the Bellman operators:
\begin{align}
\label{eq:bellman-op}
[\mathcal{T}_\pi V_\pi](s) &= \int_\mathcal{A} Q_\pi(s,a) \pi_s(da)  - \tau \text{KL}[\pi_s||\bar{\pi}_s], \\
\label{eq:bellman-optimality-op}
[\mathcal{T}V](s) &= \sup_{\pi} [\mathcal{T}_\pi V](s).
\end{align}
Equations \eqref{eq:bellman-op} and \eqref{eq:bellman-optimality-op} are analogous to the standard Bellman operator and Bellman optimality operator respectively.
Note that since the optimal policy may not be deterministic, the Bellman optimality operator selects the supremum over policies instead of actions.

We also define the soft Bellman operator
\begin{equation}
[\mathcal{L} V_\pi](s) = \tau \log \int_{\mathcal{A}} \exp \left(Q_\pi(s,a) / \tau \right) \pi_s(da).
\end{equation}
Here $\mathcal{L}$ is referred to as ``soft'', since it is a smooth approximation of the $\max$ operator.
The soft Bellman operator is connected to the Bellman optimality operator through the following result:
\begin{lemma}
\label{lemma:policy-existence}
Let $V : \mathcal{S} \to \mathbb{R}$ be a bounded measurable function. Then $\mathcal{T} V = \mathcal{L} V$ and the supremum is attained uniquely by the Boltzmann policy $\mathcal{B}[V]$ defined by
\[
\dfrac{d\mathcal{B}_s[V]}{d\bar{\pi}_s}(a) = \frac{ \exp \big(Q(s,a) / \tau \big) }
{ \int_A \exp \big( Q(s, a^\prime) / \tau \big) \bar{\pi}(da^\prime | s) }.
\]
\end{lemma}
\begin{proof}
Follows directly from \citet[Proposition~1.4.2]{dupuis11}.
\end{proof}
Analogous to the standard RL setting, we can define value and policy iteration in the entropy-regularised context, where the Bellman operators are replaced with their ``soft'' equivalents:

\begin{algorithm}[H]
   \caption{Soft Value Iteration}
   \label{alg:value-iteration}
\begin{algorithmic}
   \STATE {\bfseries Input:} MDP, temperature $\tau > 0$, bounded function $V$
   \STATE {\bfseries Output:} Optimal value function $V^*$
   \STATE initialize  $V^* \leftarrow V$
   \REPEAT
   \STATE replace  $V \leftarrow V^*$   
   \STATE apply soft Bellman operator $V^* \leftarrow \mathcal{L}[V]$
   \UNTIL{convergence}
\end{algorithmic}
\end{algorithm}
\begin{algorithm}[H]
   \caption{Soft Policy Iteration}
   \label{alg:policy-iteration}
\begin{algorithmic}
   \STATE {\bfseries Input:} MDP, temperature $\tau > 0$, proper policy $\pi$
   \STATE {\bfseries Output:} Optimal policy $\pi^*$
   \STATE initialize  $\pi^* \leftarrow \pi$
   \REPEAT
   \STATE replace  $\pi \leftarrow \pi^*$
   \STATE \textit{policy evaluation:}
   \STATE \quad find $V_\pi$, the fixed-point of $\mathcal{T}_\pi$ 
   \STATE \textit{policy improvement:}   
   \STATE \quad compute the Boltzmann policy  $\pi^* \leftarrow \mathcal{B}[V_\pi]$
   \UNTIL{convergence}
\end{algorithmic}
\end{algorithm}

Following closely along the lines of \citet{bertsekas91} and \citet{james06}, but taking special care to account for the fact that optimal policies are not necessarily deterministic, it can be shown that the above algorithms converge to optimal solutions.
\begin{theorem} \label{thm:convergence}
Suppose that Assumptions 1 and 2 \citep{james06} hold and that the optimal value function is bounded above. Then:
\begin{enumerate}[nolistsep, label=(\roman*)]
\item there exists an optimal proper policy;
\item the optimal value function is the unique bounded measurable solution to the optimality equation;
\item the soft policy iteration algorithm converges to the policy starting from any proper policy;
\item the soft value iteration algorithm converges to the optimal value function starting from any proper policy.
\end{enumerate}
\end{theorem}
%\begin{proof}
%See supplementary material.
%\end{proof}

\section{Compositionality} \label{sec:compositionality}

In lifelong learning, an agent is presented with a series of tasks drawn from some distribution.
The goal is to exploit knowledge gained in previous tasks to improve performance in the current task.
We consider an environment with fixed state space $\mathcal{S}$, action space $\mathcal{A}$, deterministic transition dynamics $\rho$, and absorbing set $\mathcal{G}$.
Let $\mathcal{D}$ be a fixed but unknown distribution over $(\mathcal{S}, \mathcal{A}, \rho, r)$.
The agent is then presented with tasks sampled from $\mathcal{D}$, which differ only in their reward functions.
In this section, we describe a compositional approach for tackling this problem.

Suppose that the reward functions drawn from $\mathcal{D}$ differ only on the absorbing set $\mathcal{G}$. This restriction was introduced by \citet{todorov09}, and is a strict subset of the \textit{successor representations} framework \citep{dayan93,barreto17}. 
Given a library of previously-solved tasks, we can combine their $Q$-functions to solve any task lying in the ``span'' of the library without further learning:
%As in prior work \citep{todorov09}, we can guarantee optimal compositionality by restricting certain aspects of the MDPs.   
%In the following, we consider a fixed domain $\mathcal{D}$ defined by state space $\mathcal{S}$, action space $\mathcal{A}$ and deterministic transition dynamics $\rho$. 
%Tasks on the domain are described completely by their reward function which are allowed to differ only on the absorbing set $\mathcal{G}$. 

%This form of reward function is a subset of those considered in the successor representations framework \citep{dayan93,barreto17}.

\begin{theorem}[Optimal Composition] \label{thm:compose}
Let $\mathcal{M}_1, \ldots, \mathcal{M}_n$ be a library of tasks drawn from $\mathcal{D}$.
Let $Q^{*, k}_\tau$ be the optimal entropy-regularised $Q$-function, and $r_k$ be the reward function for $\mathcal{M}_k$. 
Define the vectors  
\[
\mathbf{r} = [r_1, \ldots, r_n] \quad \text{ and } \quad \mathbf{Q}^*_\tau = [Q^{*, 1}_\tau, \ldots, Q^{*, n}_\tau].
\]
Given a set of non-negative weights $\mathbf{w}$, with $|| \mathbf{w} ||_1 = 1$, consider a further task drawn from $\mathcal{D}$ with reward function satisfying $r =  \tau \log\left( || \exp(\mathbf{r} / \tau) ||_{\mathbf{w}}\right)$ for all $s$ in $\mathcal{G}$, where $|| \cdot ||_\mathbf{w}$ denotes the weighted $1$-norm. 
Then the optimal $Q$-value for this task is given by:
\begin{equation} \label{eq:composed}
Q^*_\tau =  \tau \log\left( || \exp(\mathbf{Q}^*_\tau / \tau) ||_{\mathbf{w}}\right).
\end{equation}
That is, the optimal $Q$-functions for the library of tasks can be composed to form $Q^*_\tau$. 
\end{theorem}

\begin{proof}
Since $\rho$ is deterministic, we can find a measurable function $f: \mathcal{S} \times \mathcal{A} \to \mathcal{S}$ such that $\rho_{(s, a)} = \delta_{f(s, a)}$.
For any $Q$-function, define the desirability function 
\[
Z(s, a) = \exp\left( Q(s, a) / \tau \right),
\] 
and define the operator $\mathcal{U}$ on the space of non-negative bounded measurable functions by
\begin{equation*}
[\mathcal{U}Z](s, a) = \exp \left( r(s, a) / \tau \right) \int_\mathcal{A} Z(f(s, a), a) \bar{\pi}_s(da^\prime). 
\end{equation*}
We now show that the desirability function of $Q_\tau^*$ is a fixed point of $\mathcal{U}$. 
Since $V_\tau^*$ is the fixed point of the Bellman optimality operator, by combining Equation \eqref{eq:q-function}, Lemma \ref{lemma:policy-existence} and Theorem \ref{thm:convergence}, we have 
\begin{align*}
&V_\tau^*(s) = \tau \log \int_{\mathcal{A}} \exp \left (Q^*_\tau(f(s, a), a^\prime) / \tau \right) \bar{\pi}_s(da^\prime) \\
&\text{and } Q^*_\tau(s, a) = r(s, a) + V^*_\tau(f(s,a)).
\end{align*}

Then it follows that 
\begin{align*}
[\mathcal{U}Z^*_\tau](s,a) &= e^{r(s, a) / \tau} \int \exp \left ( Q^*_\tau / \tau \right) d(\rho_{(s, a)} \otimes \bar{\pi}_s) \\
&= e^{r(s, a) / \tau} \exp \left( V^*_\tau(f(s,a)) / \tau \right) = Z^*_\tau (s, a).
\end{align*}
Hence $Z^*_\tau$ is a fixed point of $\mathcal{U}$. 
Under the assumptions on the reward function $r$, the optimal $Q$-value satisfies $Q^*_\tau =  \tau \log\left( || \exp(\mathbf{Q}^*_\tau / \tau) ||_{\mathbf{w}}\right)$  on $\mathcal{G}$.
Therefore, restricted to $\mathcal{G}$, $Z^*_\tau$ is a linear combination of the desirability functions for the family of tasks.
Since \eqref{eq:composed} holds on $\mathcal{G}$ and it is clear that $\mathcal{U}$ is a linear operator, then \eqref{eq:composed} holds everywhere. 
\end{proof}

The following lemma links the previous result to the standard RL setting. 
Recall that entropy-regularisation appends a temperature-controlled penalty term to the reward function.
As the temperature parameter tends to $0$, the reward provided by the environment dominates the entropy penalty, and the problem reduces to the standard RL case:
%The next result shows that as $\tau \downarrow 0$, the optimal $Q$-functions for an entropy-regularised MDP converges to the optimal $Q$-function for the standard MDP. 

\begin{lemma} \label{lem:2}
Let $\{\tau_n\}^\infty_{n=1}$ be a sequence in $\mathbb{R}$ such that $\tau_n \downarrow 0$. 
Let $Q^*_{\tau_n}$ be the optimal $Q$-value function for $\text{MDP}(\tau_n)$: the entropy-regularised MDP with temperature parameter $\tau_n$.
Let $Q^*_{0}$ be the   optimal $Q$-value  for the standard MDP. 
Then $Q^*_{\tau_n} \uparrow Q^*_{0}$ as $n \to \infty$. 
\end{lemma}
\begin{proof}
First note that for a fixed policy $\pi$, state $s$ and action $a$, we have $Q_{\tau_n}^{\pi}(s, a) \uparrow Q_{0}^{\pi}(s, a)$ as $n \to \infty$. 
This follows directly from the definition of the entropy-regularised value function, and the fact that the KL-divergence is non-negative. 
Then using Lemma 3.14 \citep{hinderer70} to interchange the limit and supremum, we have
\begin{align*}
\lim_{n \to \infty} Q^*_{\tau_n} &= \lim_{n \to \infty} \sup_\pi Q^{\pi}_{\tau_n} = \sup_\pi \lim_{n \to \infty} Q^{\pi}_{\tau_n} \\
& = \sup_\pi Q_0^\pi = Q_0^*.
\end{align*}
Since $Q_{\tau_n}^{\pi} \uparrow Q_{0}^{\pi}$, we have $Q^*_{\tau_n} \uparrow Q^*_{0}$ as $n \to \infty$. 
\end{proof}

Finally, we show that composition holds in the standard RL setting by taking the low-temperature limit of Theorem 2.
\begin{corollary} \label{cor:1}
Let $\{\tau_n\}^\infty_{n=1}$ be a sequence in $\mathbb{R}$ such that $\tau_n \downarrow 0$. 
Then $\max \mathbf{Q^{*}_{\tau_n}} \uparrow Q^*_0$ as $n \to \infty$.
\end{corollary}

\begin{proof}
For a fixed state $s$ and action $a$ and a possible reordering of the vector $\mathbf{Q^{*}_{0}}(s,a)$, we may suppose, without loss of generality, that $Q^{*, 1}_0(s, a) = \max \mathbf{Q^{*}_0}(s, a) $.
Then by Lemma~\ref{lem:2}, we can find an $N$ in $\mathbb{N}$ such that 
\[
Q^{*, 1}_{\tau_n}(s, a) = \max \mathbf{Q^{*}_{\tau_n}}(s, a) \text{ for all $n \geq N$}.
\] 
Since $\log$ is continuous, we have from Theorem \ref{thm:compose} that
\begin{align*}
\lim_{n \to \infty} Q^{*}_{\tau_n} &= \log \left( \lim_{n \to \infty} || \exp(\mathbf{Q^{*}_{\tau_n}})||_\mathbf{w}^{1 / \tau_n} \right), 
\end{align*}
where $|| \cdot ||^p_\mathbf{w}$ denotes the weighted $p$-norm.
By factoring $\exp(Q^{*, 1}_{\tau_n})$ out of $|| \exp(\mathbf{Q^{*}_{\tau_n}})||_\mathbf{w}^{1 / \tau_n}$, we are left with
\begin{align*}
|| 1, \exp(\Delta_2 ), \ldots, \exp(\Delta_k )||_\mathbf{w}^{1 / \tau_n},
\end{align*}
where $\Delta_i = Q^{*, i}_{\tau_n} -  Q^{*, 1}_{\tau_n}$ for $i = 2, \ldots, k$.
Since $Q^{*, 1}_{\tau_n}(s, a)$ is the maximum of  $\mathbf{Q^{*}_{\tau_n}}(s, a)$ for all $n \geq N$, the limit as $n \to \infty$  of the above is $1$.
Then it follows that 
\begin{align*}
\lim_{n \to \infty} Q^{*, 1}_{\tau_n}(s, a) &= \log \left( \lim_{n \to \infty} \exp(Q^{*, 1}_{\tau_n}(s, a)) \right) \\
&= Q^{*, 1}_0(s, a).
\end{align*}
Since $s$ and $a$ were arbitrary and  $Q^{*, m}_{\tau_n} \uparrow Q^{*, m}_0$, we have that $\max \mathbf{Q^{*}_{\tau_n}} \uparrow Q^*_0$ as $n \to \infty$.
\end{proof}

Comparing Theorem~\ref{thm:compose} to Corollary~\ref{cor:1}, we see that as the temperature parameter decreases to zero, the weight vector has less influence on the composed $Q$-function.
In the limit, the optimal $Q$-function is independent of the weights and is simply the maximum of the library functions.
This suggests a natural trade-off between our ability to interpolate between $Q$-functions, and the stochasticity of the optimal policy.
Furthermore, Corollary~\ref{cor:1} mirrors that of \textit{generalised policy improvement} \citep{barreto17}, which shows that computing the maximum of a set of $Q$-functions results in an improved $Q$-function. 
In our case, the resulting $Q$-function is not merely an improvement, but is in fact optimal. 

The composition described in this section can be viewed as an --OR-- task composition: if objectives of two tasks are to achieve goals $A$ and $B$ respectively, then the composed $Q$-function will achieve $A$--OR--$B$ optimally. 
\citet{haarnoja18} show that an approximate --AND-- composition is also possible for entropy-regularised RL.
That is, if the goals $A$ and $B$ partially overlap, the composed $Q$-function will achieve $A$--AND--$B$ approximately.
The idea is that the optimal $Q$-function for the composite task can be approximated by the average of the library $Q$-functions.
We include their results for completeness:

\begin{lemma}[{\citealp{haarnoja18}}] \label{lem:haarnoja1}
Let $Q^{*, 1}_\tau$ and $Q^{*, 2}_\tau$ be the optimal $Q$-functions for two tasks drawn from $\mathcal{D}$ with rewards $r_1$ and $r_2$. 
Define the averaged $Q$-function $Q_{\text{ave}} := (Q^{*, 1}_\tau + Q^{*, 2}_\tau) / 2$.
Then the optimal $Q$-function $ Q^*_\tau$ for the task with reward function $r = (r_1 + r_2) / 2$ satisfies
\[
Q_{\text{ave}} \geq Q^*_\tau \geq Q_{\text{ave}} - C^*_\tau,
\] 
where $C^*_\tau$ is a fixed point of 
\[
\tau \mathbb{E}_{s^\prime \sim \rho(s, a)} \left[ D_\frac{1}{2} \left(\pi_s^{*, 1} || \pi_s^{*, 2}\right)  + \max_{a^\prime}C(s^\prime, a^\prime) \right], 
\]
the policy $\pi_s^{*, i}$ is the optimal Boltzmann policy for task $i$, and $ D_\frac{1}{2} ( \cdot || \cdot)$ is the R\'enyi divergence of order $\frac{1}{2}$.
\end{lemma}

\begin{theorem}[{\citealp{haarnoja18}}]
Using the definitions in Lemma \ref{lem:haarnoja1}, the value of the composed policy $\pi^{\text{ave}}$ satisfies 
\[
Q_{\pi^{\text{ave}}} \geq Q^*_\tau - F^*_\tau,
\]
where $F^*_\tau$ is a fixed point of 
\[
\tau \mathbb{E}_{s^\prime \sim \rho(s, a)} \left[  \mathbb{E}_{a^\prime \sim \pi^{\text{ave}}_{s^\prime}} \left[ C^*_\tau(s^\prime, a^\prime) - F(s^\prime, a^\prime)   \right] \right].
\]
\end{theorem}

We believe that the low-temperature result from Lemma \ref{lem:2} can be used to obtain similar results for the standard RL framework.
We provide empirical evidence of this in the next section, and leave a formal proof to future work.

\section{Experiments}

To demonstrate composition, we perform a series of experiments in a grid-world video game (Figure~\ref{fig:map}).
The goal of the game is to collect items of different colours and shapes. 
The agent has four actions that move it a single step in any of the cardinal directions, unless it collides with a wall.%, in which case it remains in place.
Each object in the domain is one of two shapes (squares and circles), and one of three colours (blue, beige and purple), for a total of six objects (see Figure~\ref{fig:items}).

\begin{figure}[h!]
	\centering
   	\begin{subfigure}[t]{0.45\linewidth}
   		\centering
    	\includegraphics[height=3.0cm]{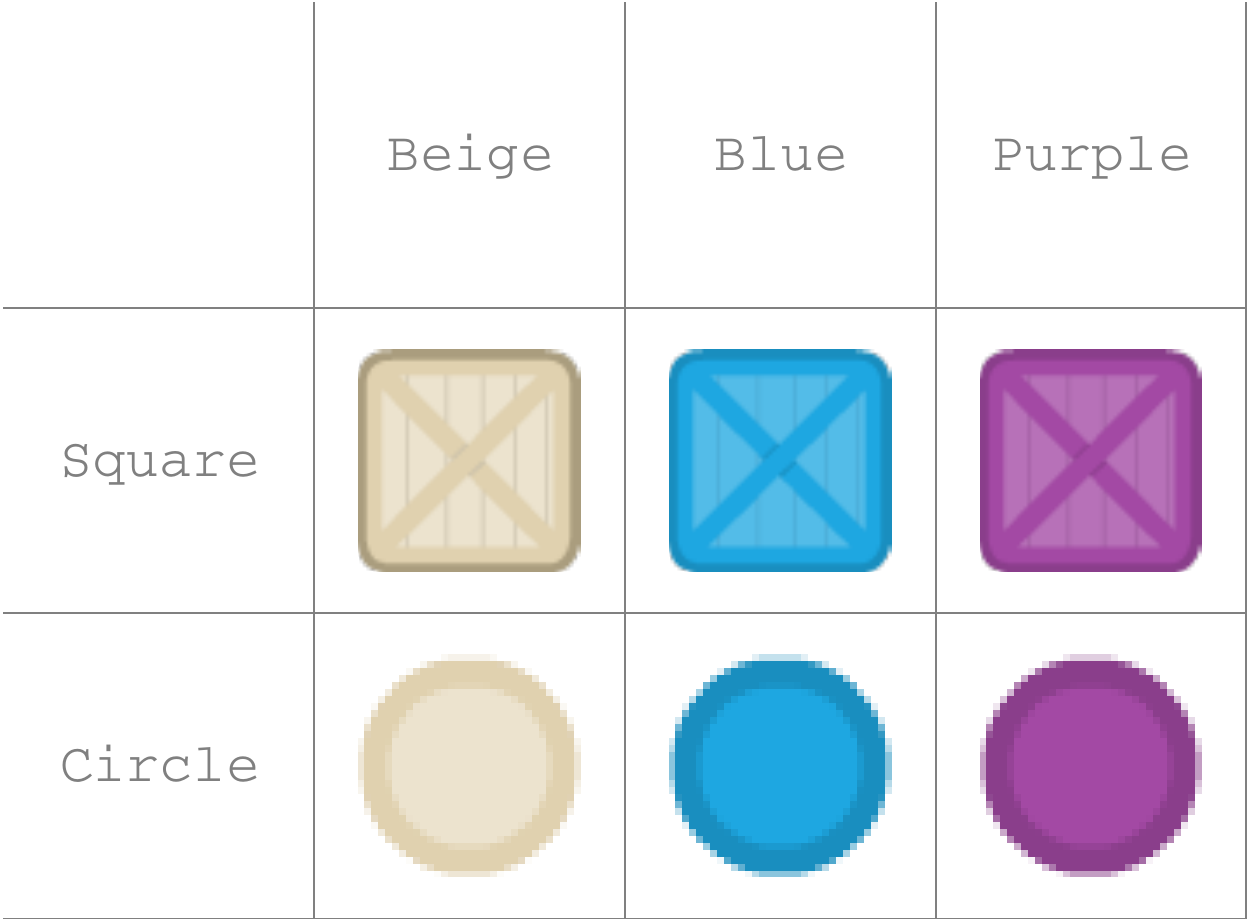}
    	\caption{Items to be collected.} \label{fig:items}
	\end{subfigure}%
	~~	
	\begin{subfigure}[t]{0.45\linewidth}
		\centering
      	\includegraphics[height=3.0cm]{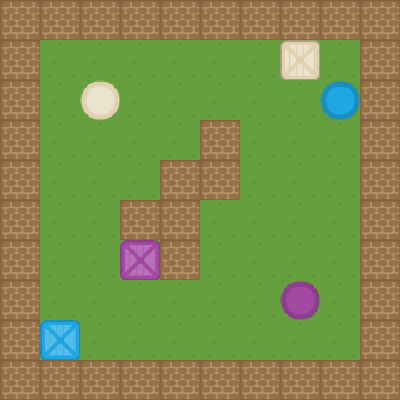}
      	\caption{Layout of the grid-world.} \label{fig:map}
	\end{subfigure}%
	\caption{}
\end{figure}

We construct a number of different tasks based on the objects that the agent must collect, the task's name specifying the objects to be collected. 
For example, \texttt{Purple} refers to the task where an agent must collect any purple object, while \texttt{BeigeSquare} requires collecting the single beige square.

\begin{figure*}[b!]
	\centering
   	\begin{subfigure}[m]{0.33\textwidth}
   		\centering
    	\includegraphics[height=3.4cm]{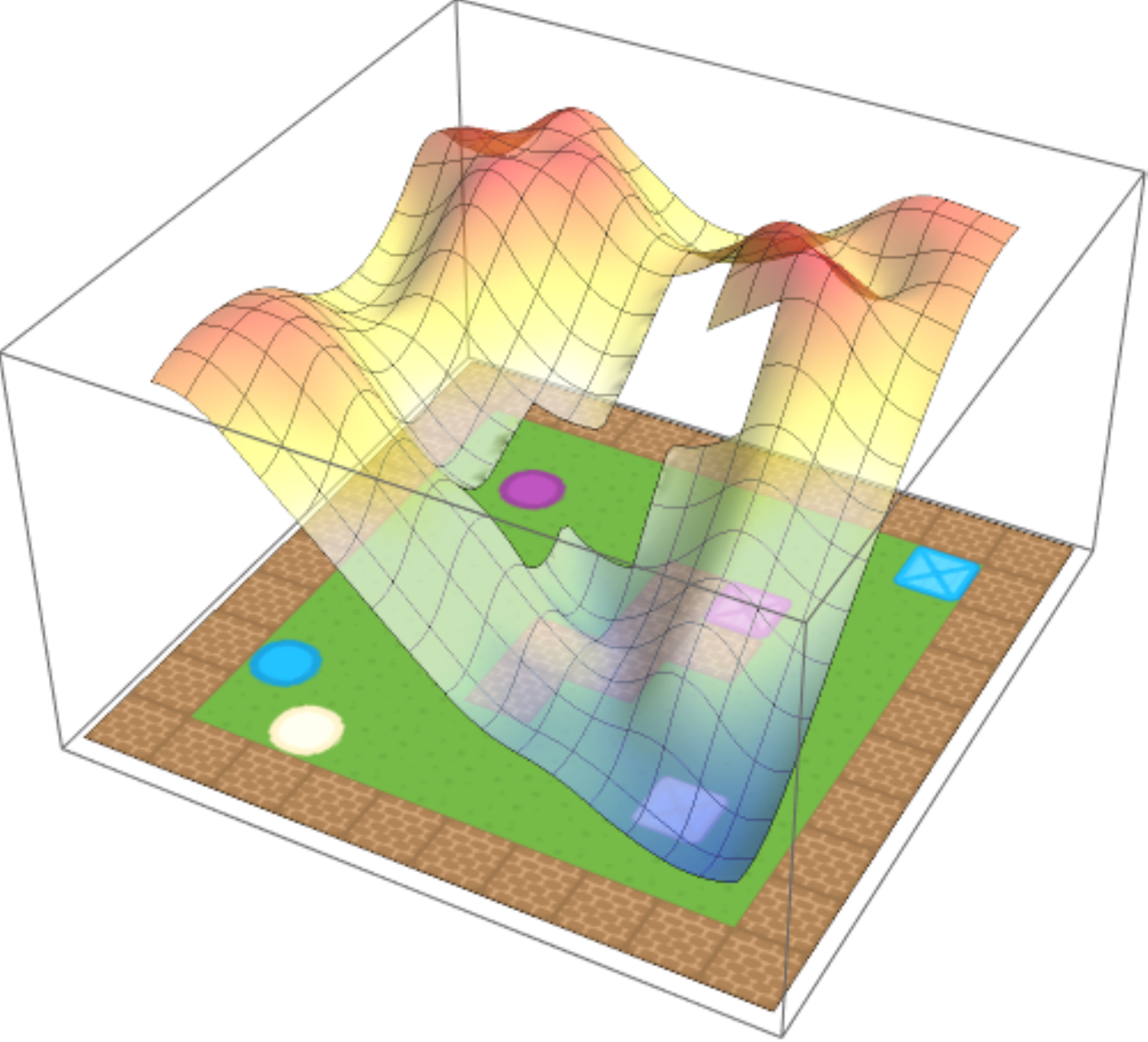}
    	\caption{} \label{fig:value_or}
	\end{subfigure}%
	~~	
	\begin{subfigure}[m]{0.33\textwidth}
		\centering
      	\includegraphics[height=3.0cm]{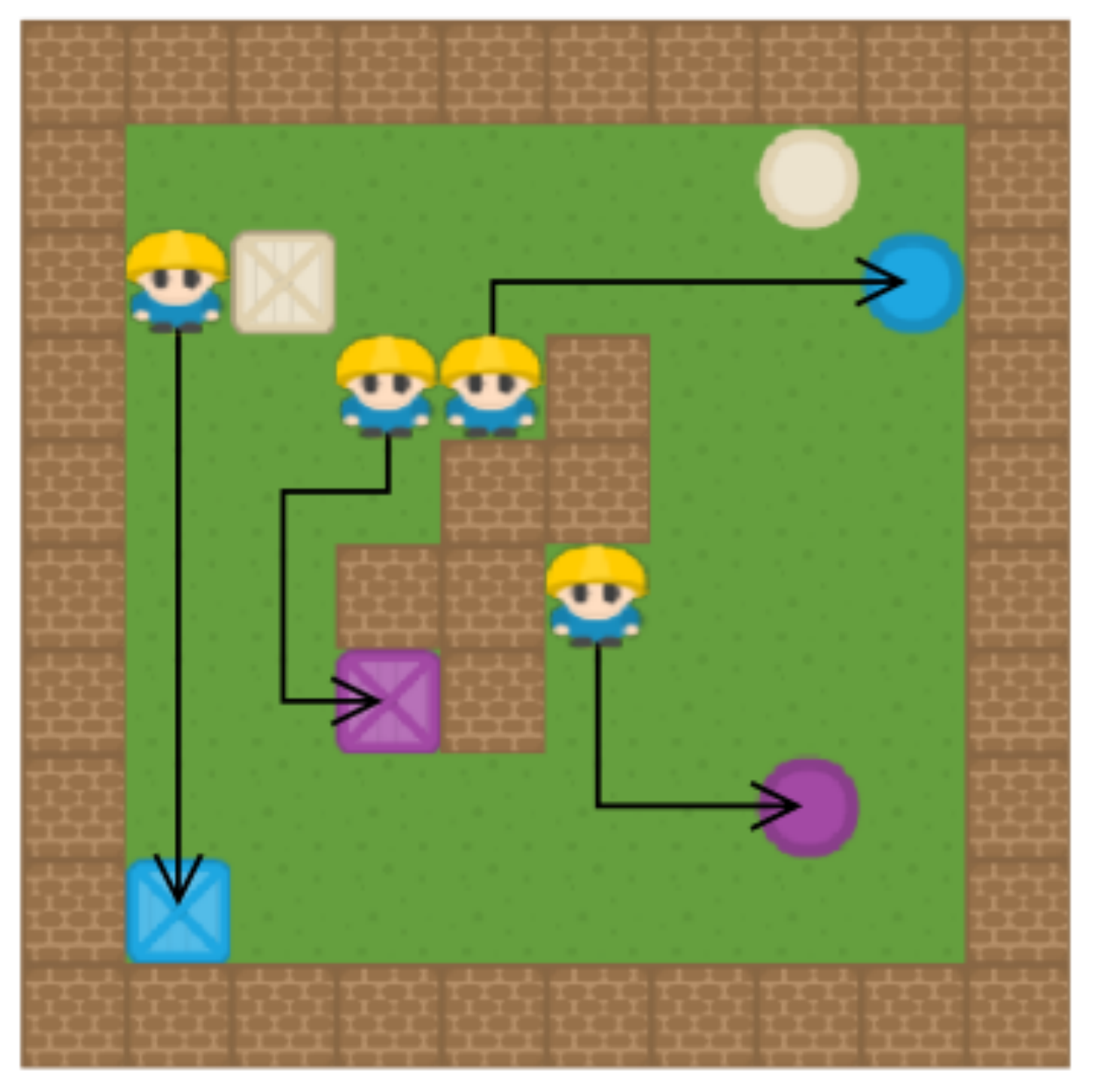}
      	\caption{} \label{fig:trajectories_or}
	\end{subfigure}%
	\begin{subfigure}[m]{0.33\textwidth}
		\centering
      	\includegraphics[height=3.4cm]{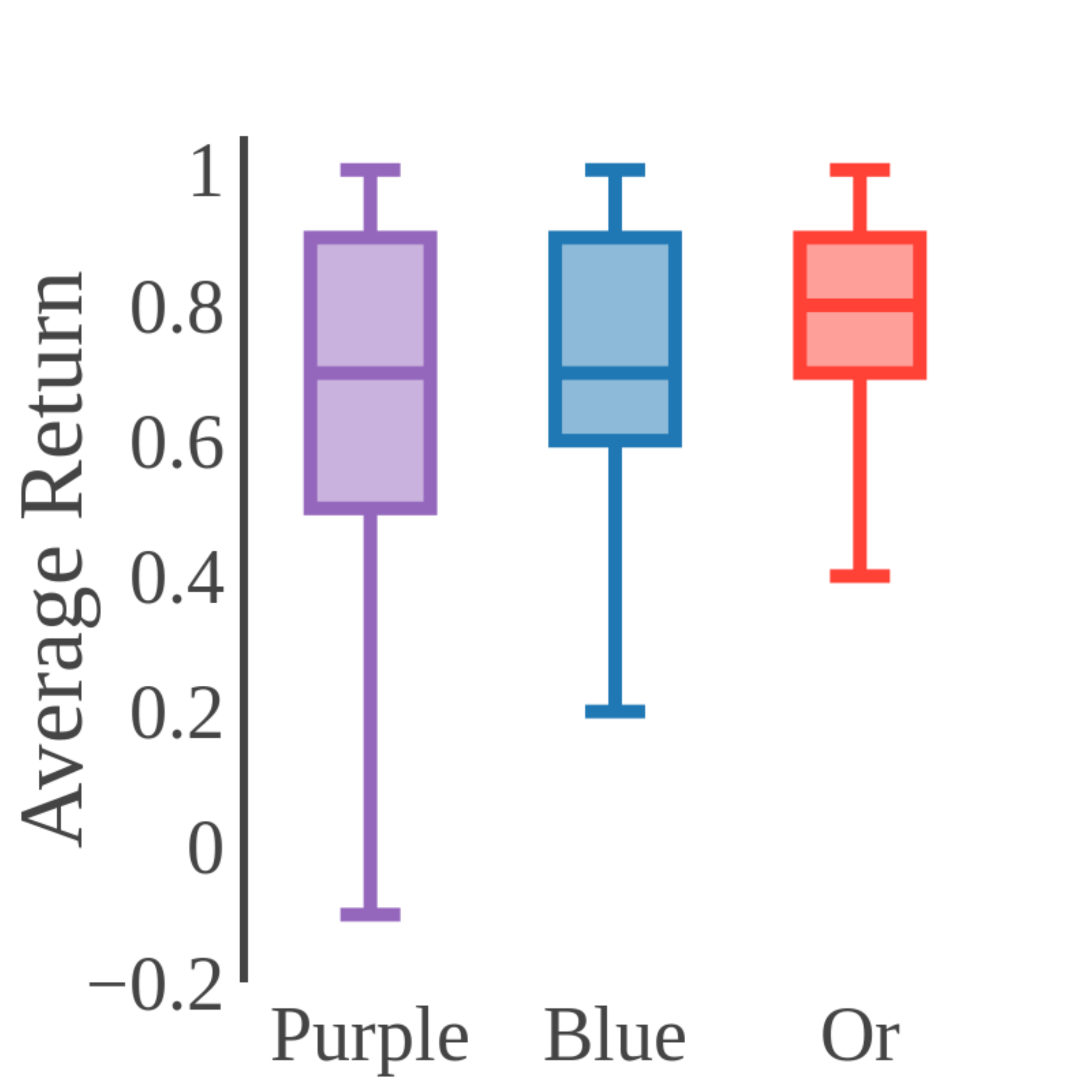}
  	    \caption{} \label{fig:rewards_or}
	\end{subfigure}%
	\caption{(\subref{fig:value_or}) The optimal value function for \texttt{PurpleOrBlue}, which is produced by composing the \texttt{Purple} and \texttt{Blue} $Q$-functions. The multimodality of the composite value function is clearly visible. (\subref{fig:trajectories_or}) Sample trajectories for the composite \texttt{PurpleOrBlue} task, with the agent beginning at different positions. The agent selects the shortest path to any of the target objects. (\subref{fig:rewards_or}) Returns from $50$k episodes. The first two box plots are the results of acting in the \texttt{PurpleOrBlue} task using only one of the base $Q$-functions, while the third uses the composite $Q$-function.}\label{fig:or}
\end{figure*}

For each task, the episode begins by randomly positioning the six objects and the agent. 
At each timestep, the agent receives a reward of $-0.1$. 
If the correct object is collected, the agent receives a reward of $1$ and the episode terminates.
We first learn to solve a number of base tasks using (soft) deep $Q$-learning \citep{mnih15,schulman17}, where each task is trained with a separate network. 
The resulting networks are collected into a library from which we will later compose new $Q$-functions.

The input to our network is a single RGB frame of size $84 \times 84$, which is passed through three convolutional layers and two fully-connected layers before outputting the predicted Q-values for the given state.
%\footnote{A full description of the architecture is provided in the supplementary material.} 
Using the results in Section~\ref{sec:compositionality}, we compose optimal $Q$-functions from those in the library.

\subsection{--OR-- Composition} 
 
Here we consider new tasks that can be described as the union of a set of base tasks in the standard RL setting. 
%This is semantically equivalent to the --OR-- task construction. 
We train an agent separately on the \texttt{Purple} and \texttt{Blue} tasks, adding the corresponding $Q$-functions to our library. 
We use Corollary \ref{cor:1} to produce the optimal $Q$-function for the composite  \texttt{PurpleOrBlue} task, which requires the agent to pick up either blue or purple objects, without any further learning.  
Results are given in Figure~\ref{fig:or}.
 
The local maxima over blue and purple objects illustrates the multimodality of the value function (Figure~\ref{fig:value_or}).
This is similar to approaches such as soft $Q$-learning \citep{haarnoja17}, which are also able to learn  multimodal policies.
However, we have observed that directly learning a truly multimodal policy for the composite task can be difficult.
If the entropy regularisation is too high, the resulting policy is extremely stochastic.
Too low a temperature results in a loss of multimodality, owing to winner's curse.
It is instead far easier to learn unimodal value functions for each of the base tasks, and then compose them to produce optimal multimodal value functions.   
 
%This particular form of composition is provably optimal (ignoring the error inherent in function approximation); it is thus unsurprising to observe that the agent selects the shortest path to any of the target objects (Figure~\ref{fig:trajectories_or}).
%The LMDP framework provides the same optimality guarantees \citep{todorov09}, but has never been extended beyond the tabular case. 

\begin{figure*}[t!]
	\centering
   	\begin{subfigure}[m]{0.23\textwidth}
   		\centering
    	\includegraphics[height=3.2cm]{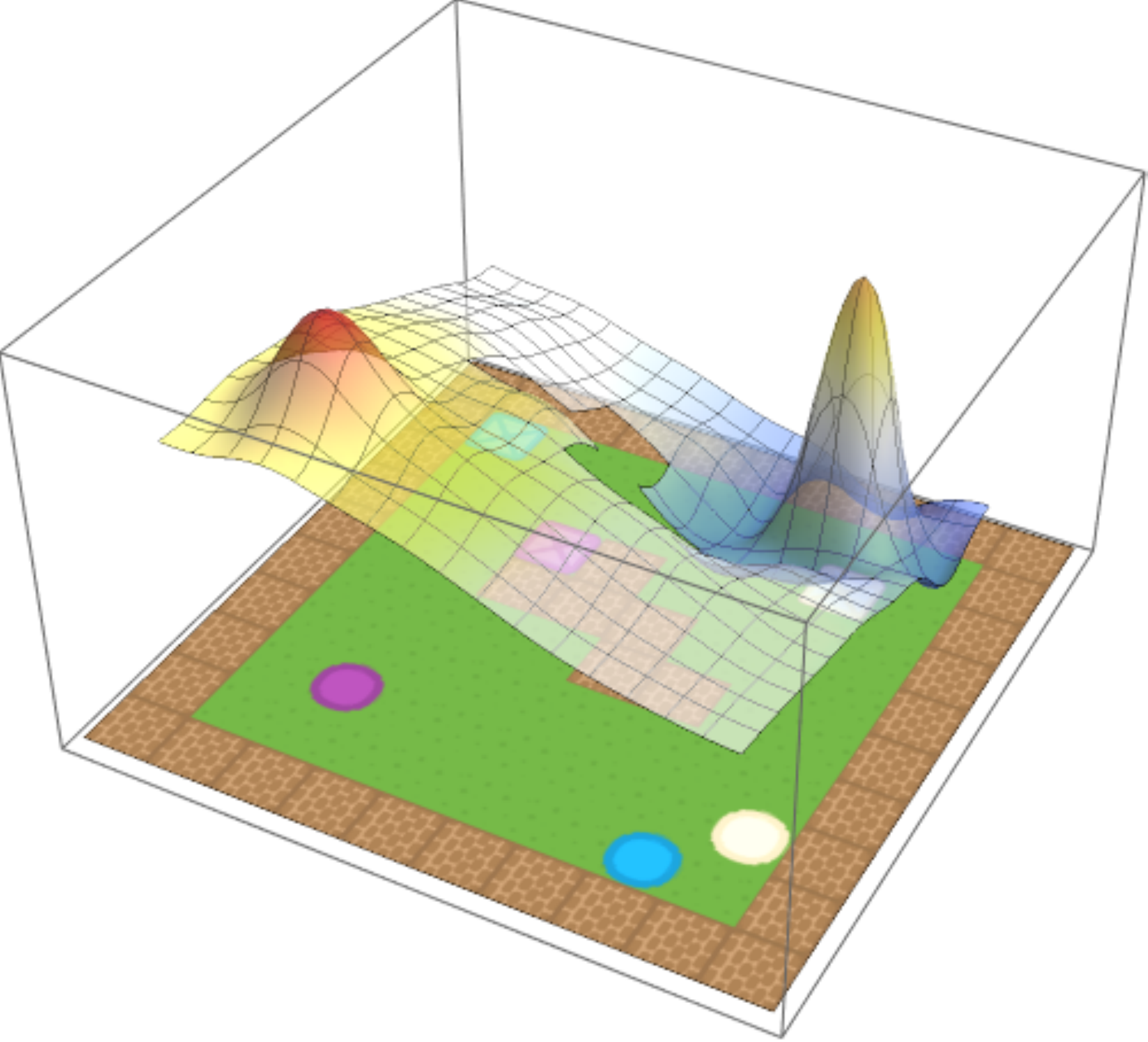}
    	\caption{\texttt{BeigeSquare}: $0.0$} \label{fig:plot_weighted_or_000}
	\end{subfigure}%
	~~	
	\begin{subfigure}[m]{0.23\textwidth}
		\centering
      	\includegraphics[height=3.2cm]{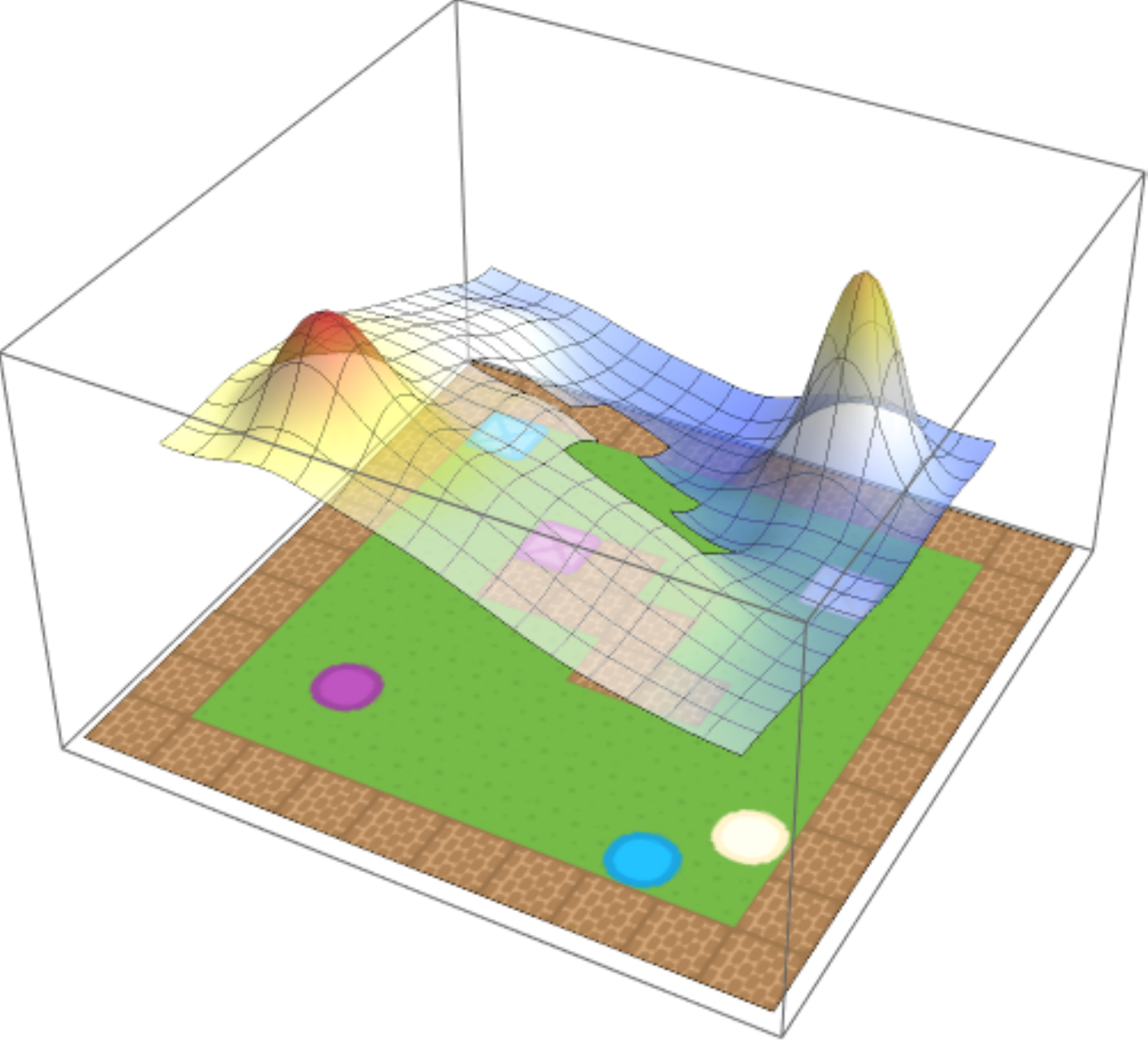}
      	\caption{\texttt{BeigeSquare}: $0.05$} \label{fig:plot_weighted_or_005}
	\end{subfigure}%
	\begin{subfigure}[m]{0.23\textwidth}
		\centering
      	\includegraphics[height=3.2cm]{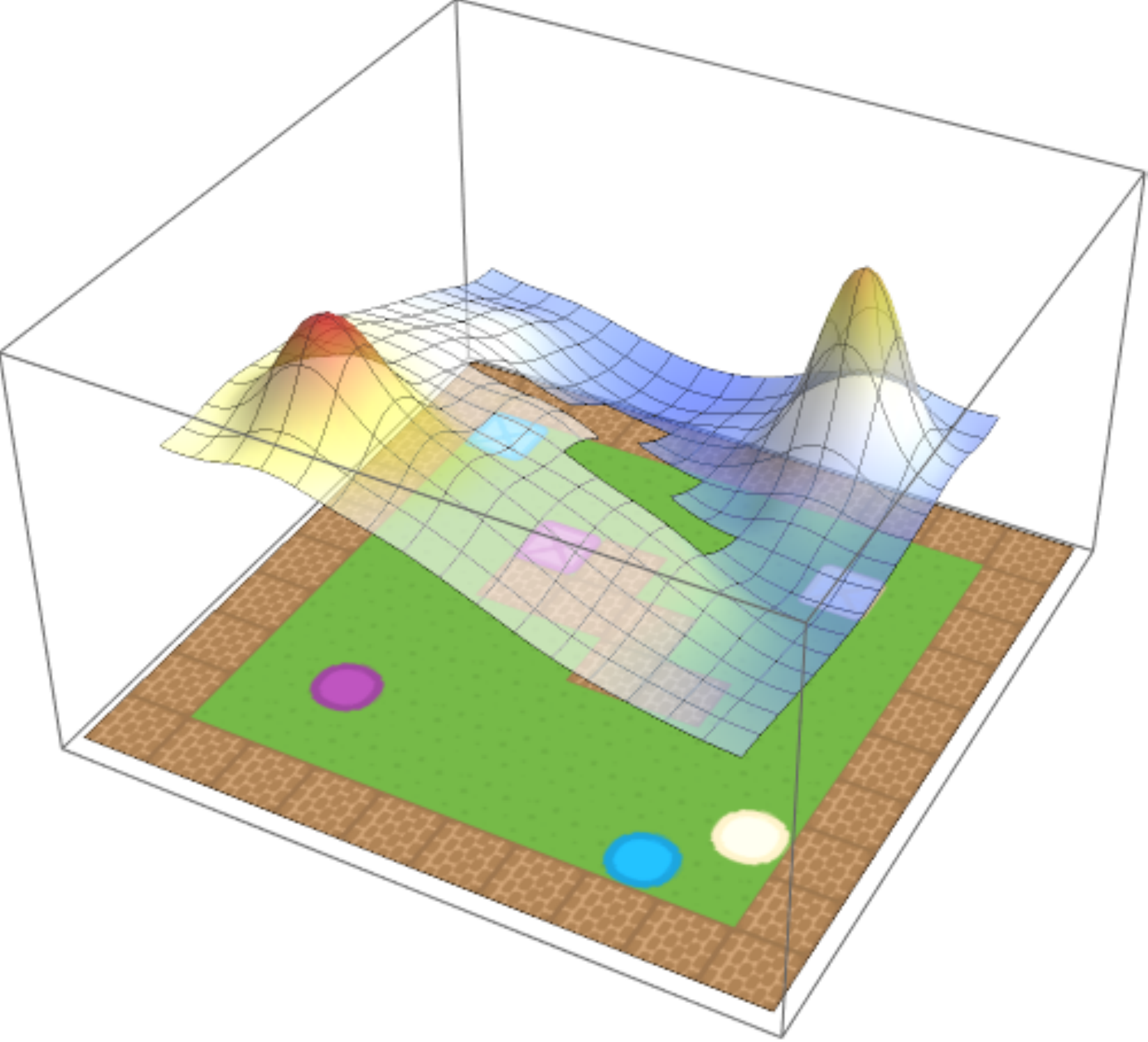}
  	    \caption{\texttt{BeigeSquare}: $0.1$} \label{fig:plot_weighted_or_010}
	\end{subfigure}
	\begin{subfigure}[m]{0.23\textwidth}
		\centering
      	\includegraphics[height=3.2cm]{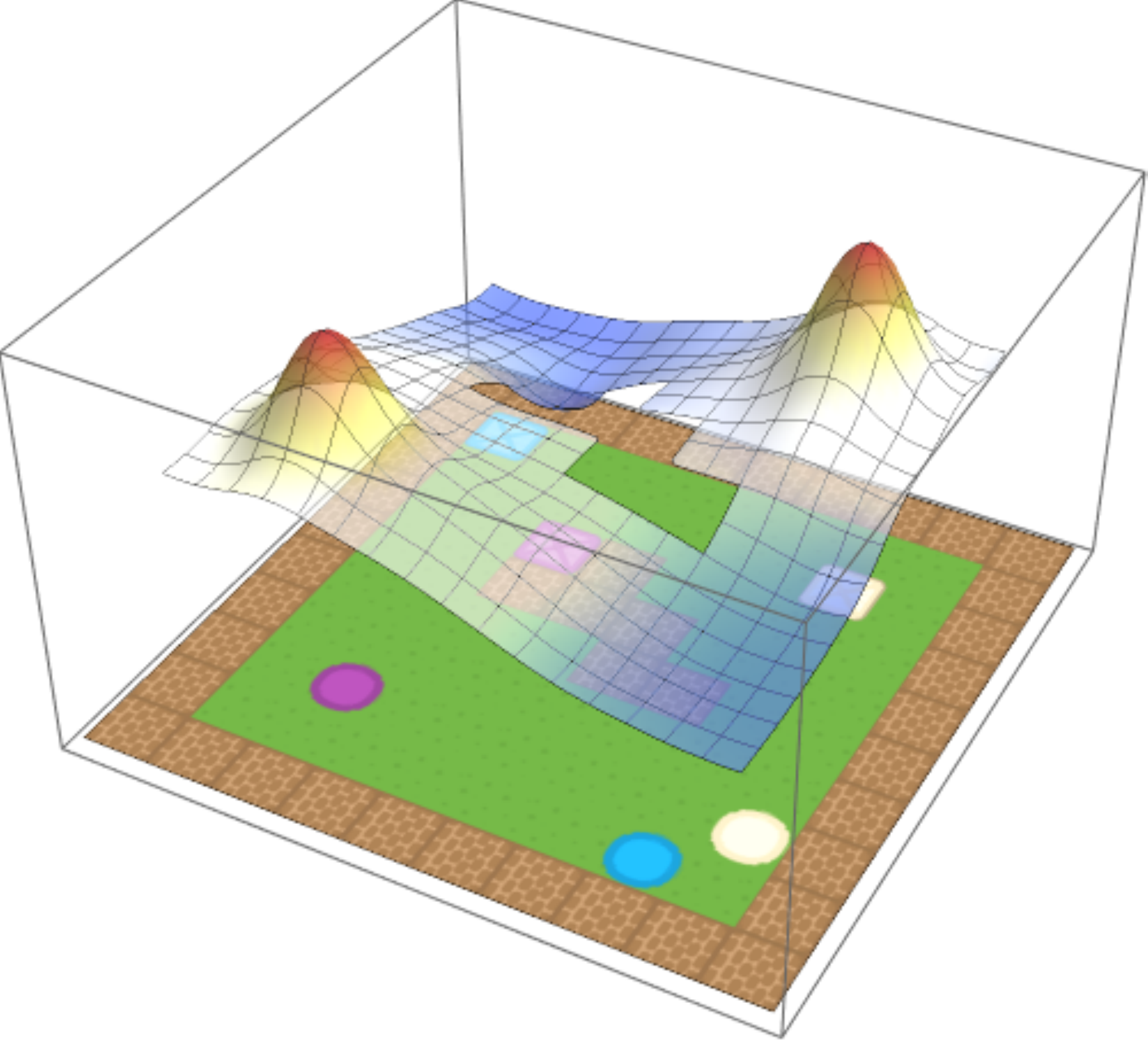}
  	    \caption{\texttt{BeigeSquare}: $0.5$} \label{fig:plot_weighted_or_050}
	\end{subfigure}
	\\
	   	\begin{subfigure}[m]{0.23\textwidth}
   		\centering
    	\includegraphics[height=3.2cm]{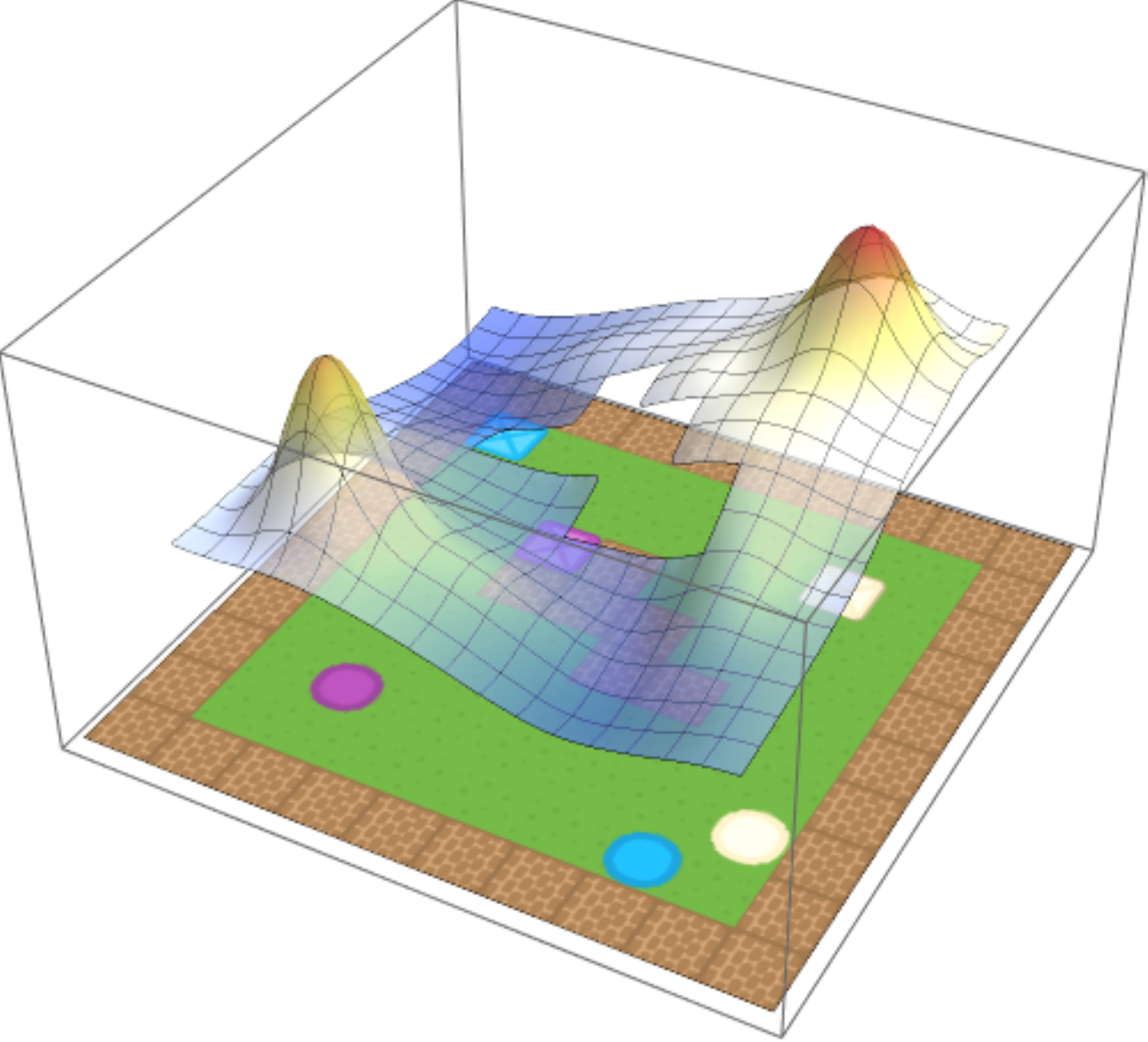}
    	\caption{\texttt{BeigeSquare}: $0.9$} \label{fig:plot_weighted_or_090}
	\end{subfigure}%
	~~	
	\begin{subfigure}[m]{0.23\textwidth}
		\centering
      	\includegraphics[height=3.2cm]{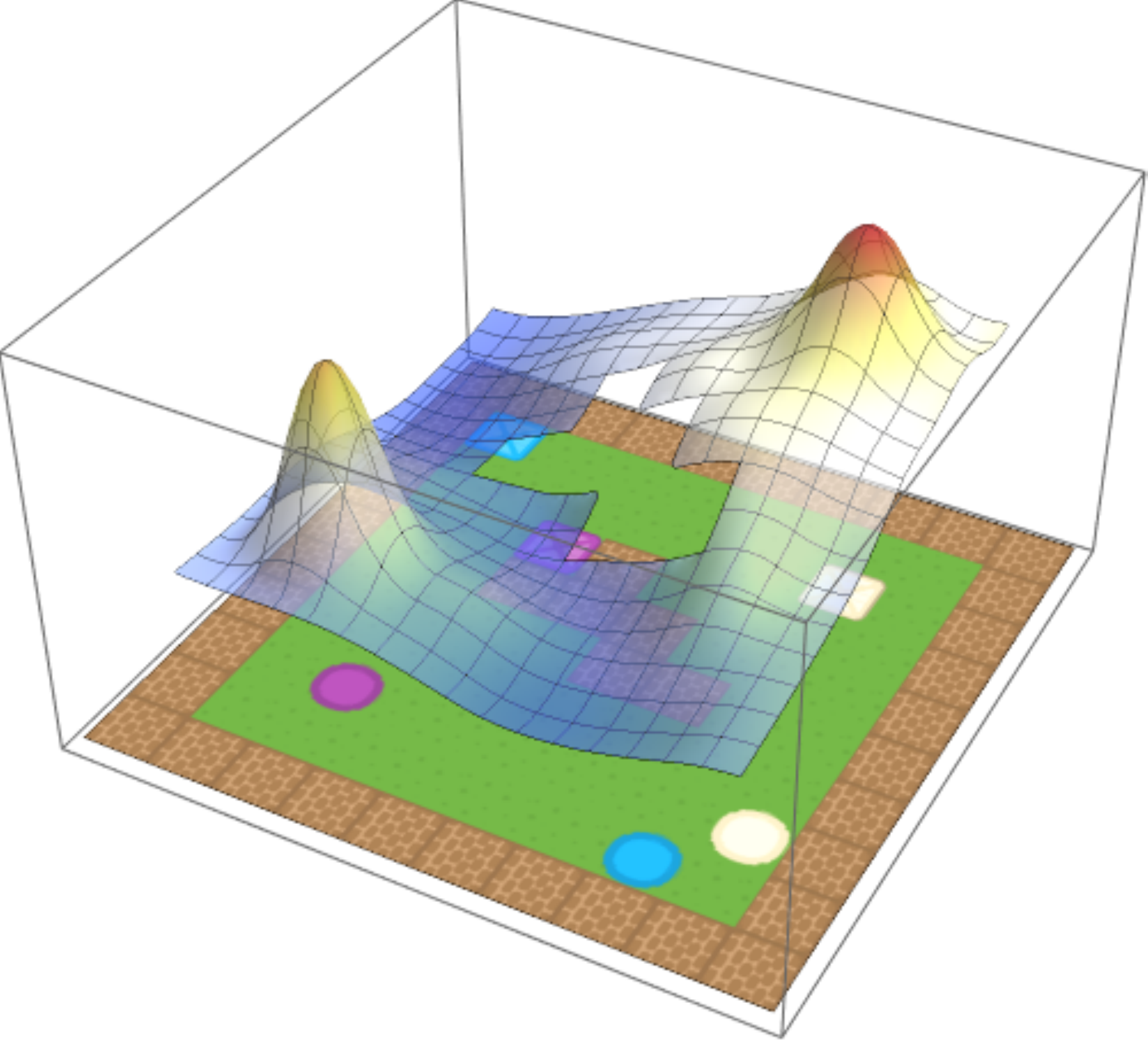}
      	\caption{\texttt{BeigeSquare}: $0.95$} \label{fig:plot_weighted_or_095}
	\end{subfigure}%
	\begin{subfigure}[m]{0.23\textwidth}
		\centering
      	\includegraphics[height=3.2cm]{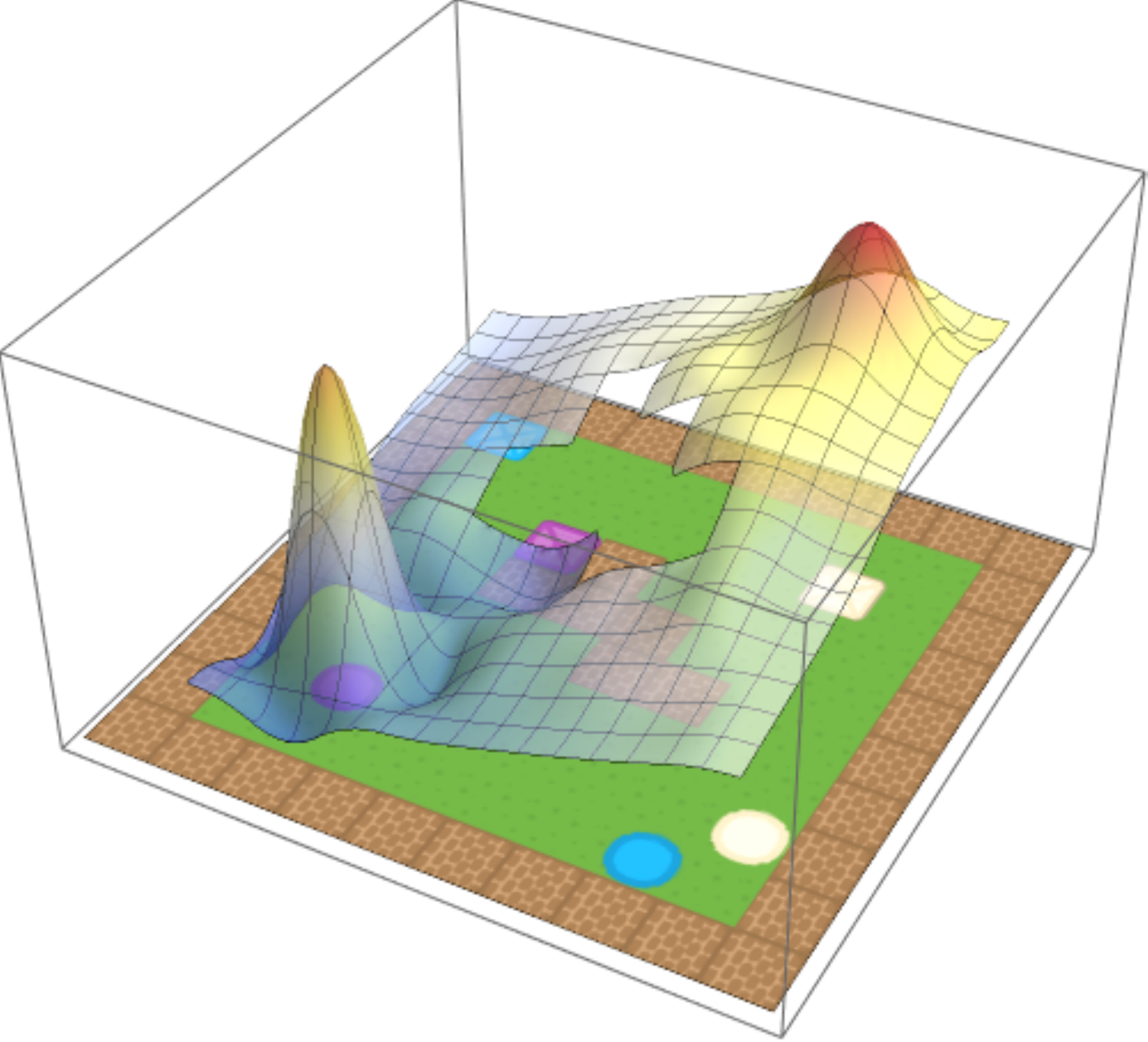}
  	    \caption{\texttt{BeigeSquare}: $1.0$} \label{fig:plot_weighted_or_100}
	\end{subfigure}%
	\begin{subfigure}[m]{0.23\textwidth}
		\centering
      	\includegraphics[height=3.2cm]{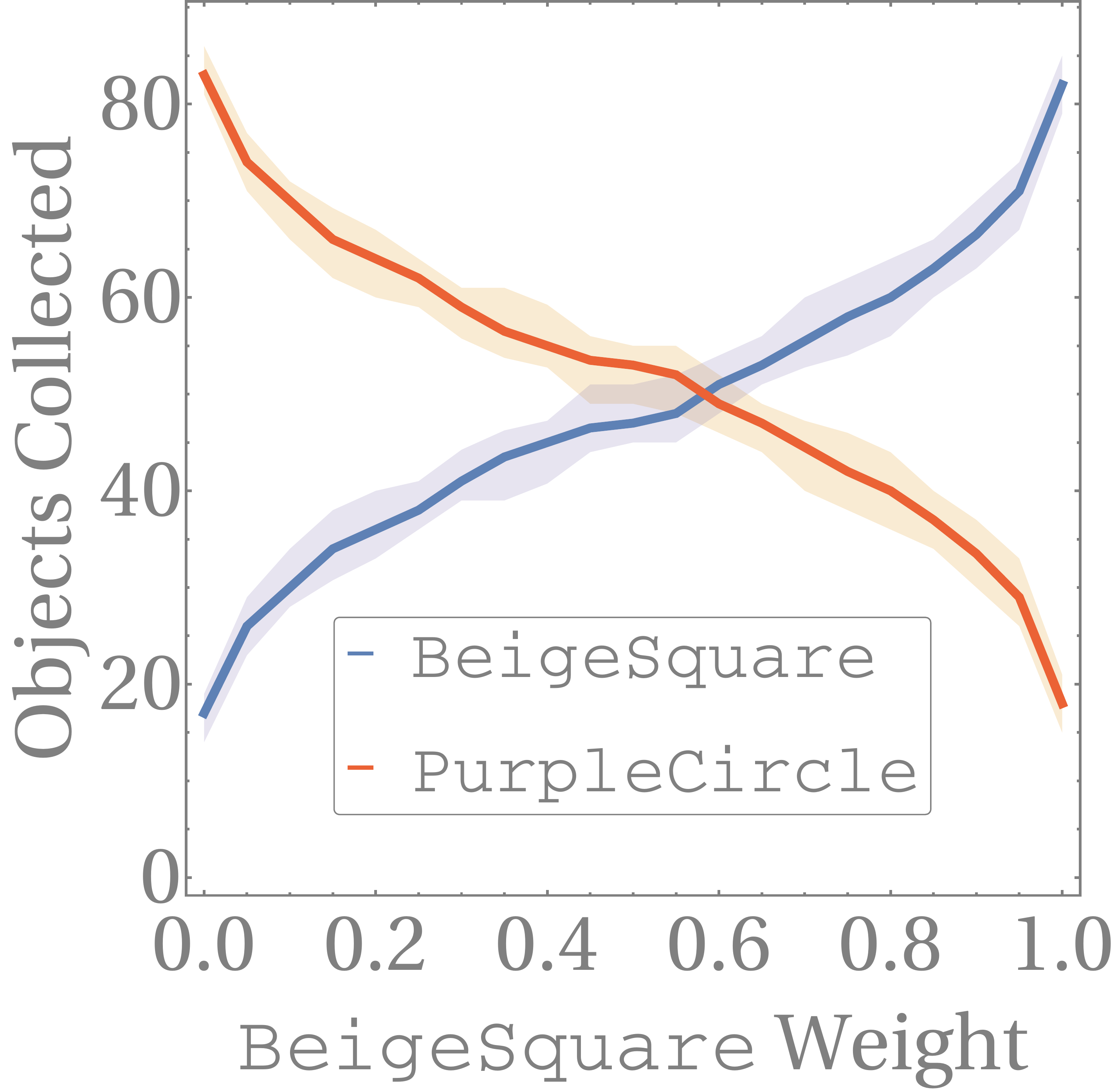}
  	    \caption{} \label{fig:weighted_or_collected}
	\end{subfigure}%		

	\caption{(\subref{fig:plot_weighted_or_000}--\subref{fig:plot_weighted_or_100}) Weighted composed value function for the task \texttt{BeigeSquareOrPurpleCircle}.  The weight assigned to the $Q$-function for \texttt{BeigeSquare} is varied from $0$ to $1$. (\subref{fig:weighted_or_collected}) The number of beige squares compared to purple circles collected by the agent as the weights are varied in steps of $0.05$. Results for each weight were averaged over $80$ runs of $100$ episodes.}\label{fig:weighted_or}
\end{figure*}

\begin{figure*}[b!]
	\centering
   	\begin{subfigure}[m]{0.33\textwidth}
   		\centering
    	\includegraphics[height=3.4cm]{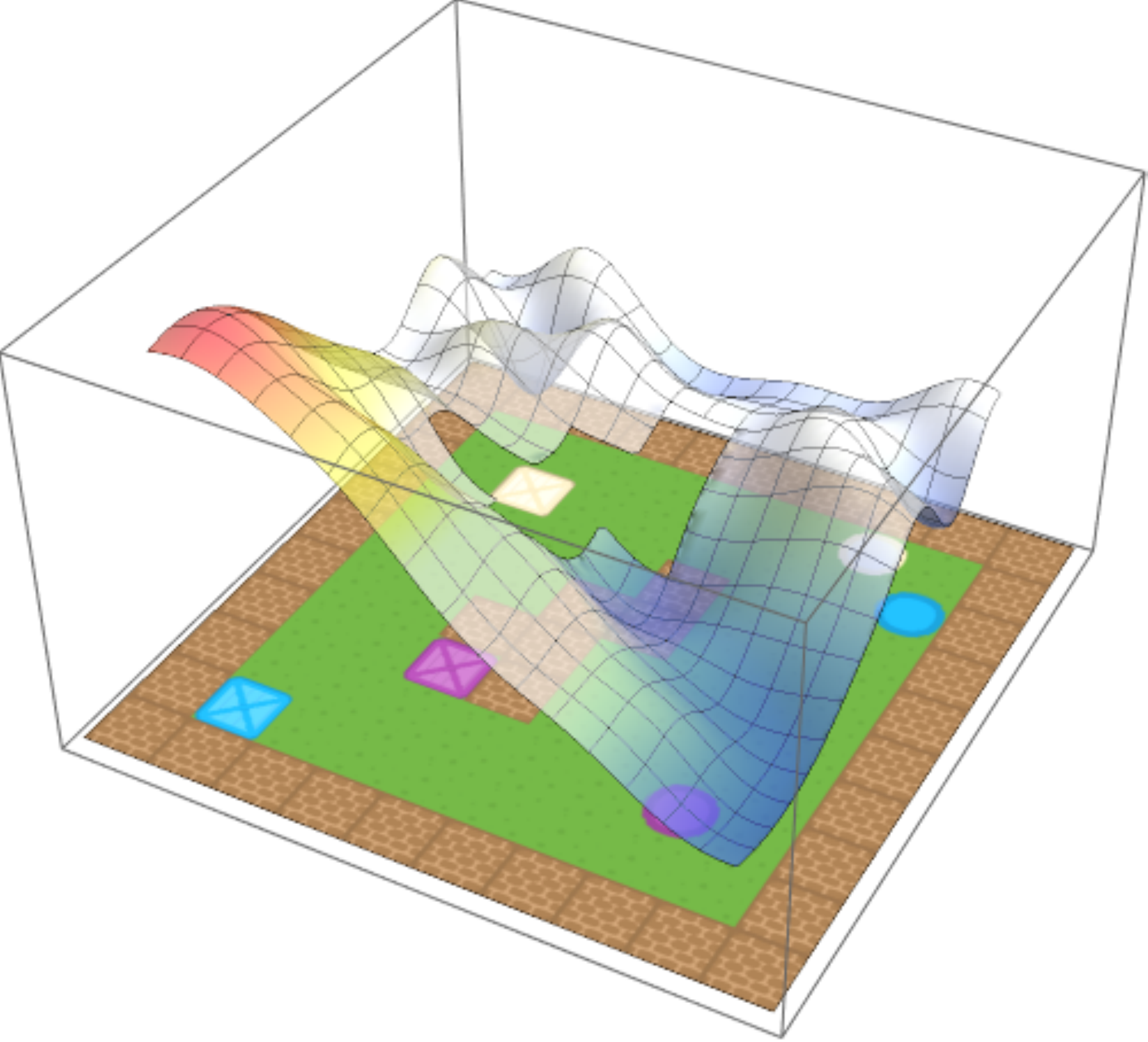}
    	\caption{} \label{fig:value_and}
	\end{subfigure}%
	~~	
	\begin{subfigure}[m]{0.33\textwidth}
		\centering
      	\includegraphics[height=3.0cm]{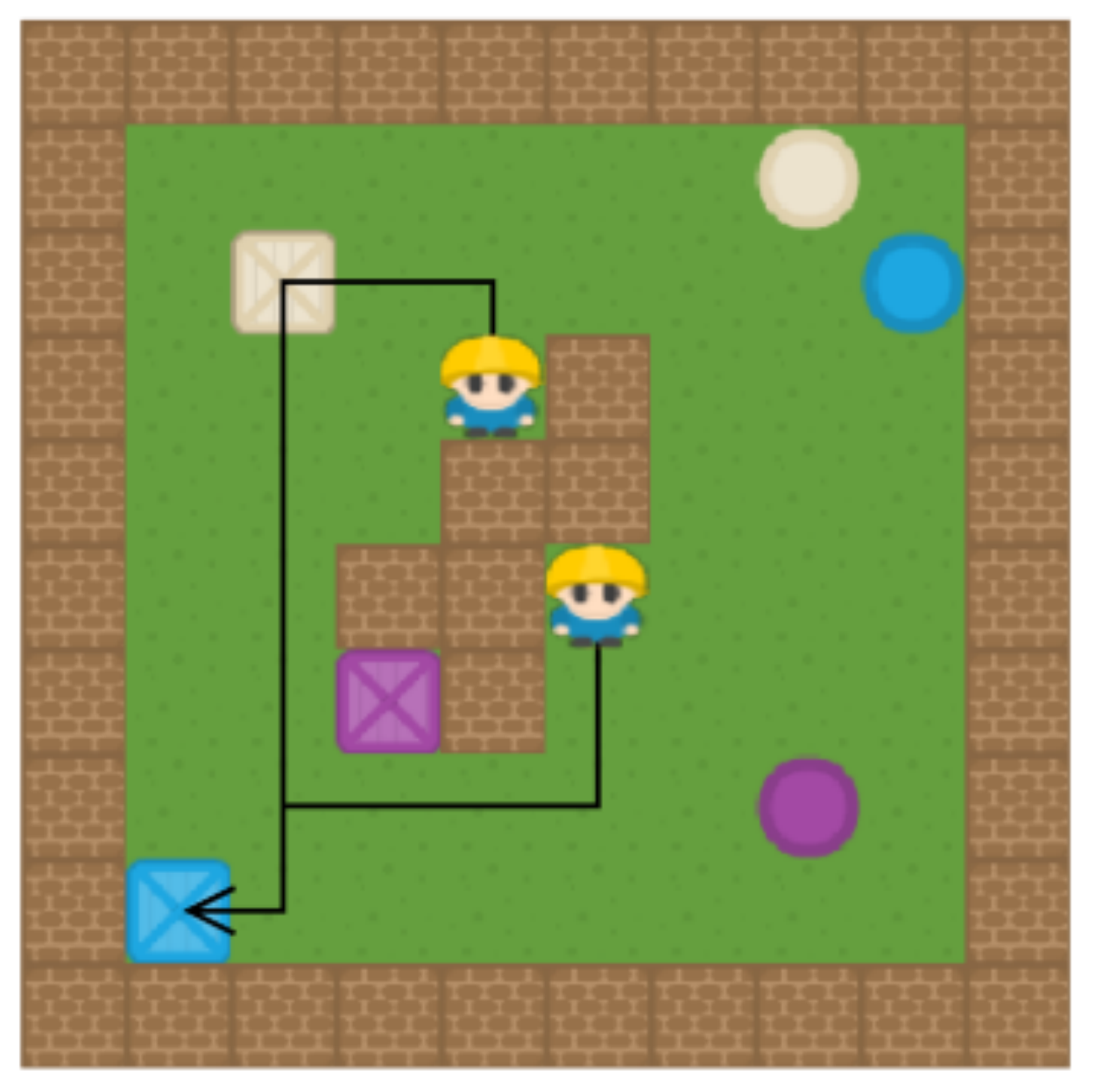}
      	\caption{} \label{fig:trajectories_and}
	\end{subfigure}%
	\begin{subfigure}[m]{0.33\textwidth}
		\centering
      	\includegraphics[height=3.4cm]{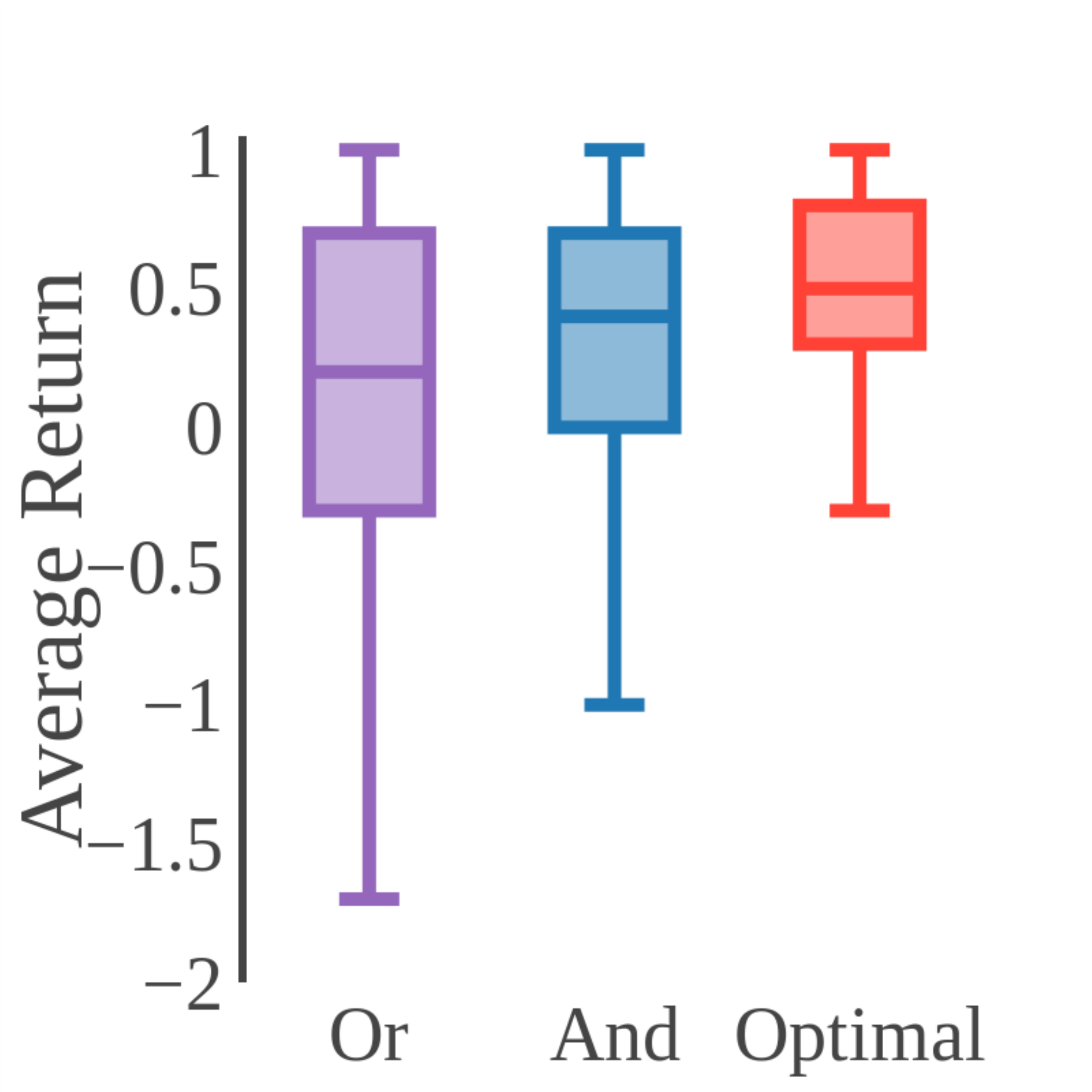}
  	    \caption{} \label{fig:rewards_and}
	\end{subfigure}%
	\caption{(\subref{fig:value_and}) The approximately optimal value function of the composed policies. Local optima are clearly visible. (\subref{fig:trajectories_and})  Sample trajectories from the composed policy beginning from different starting positions. The agent exhibits suboptimal, but sensible behaviour near beige squares. (\subref{fig:rewards_and}) The IQR of returns from $50$k episodes. The first box plot is the return from the optimal solution to the \textit{union} of tasks, the second is the result of the approximate \textit{intersection} of tasks, and the third is the true optimal policy.} \label{fig:and}
\end{figure*}

\subsection{Linear Task Combinations}

In Theorem~\ref{thm:compose} we showed that in the entropy-regularised setting, the composed $Q$-function is dependent on a weight vector $\mathbf{w}$.
This allows us to achieve a more general type of composition. % than that demonstrated in the previous experiment.
In particular, we can immediately compute any optimal $Q$-function that lies in the ``span'' of the library $Q$-functions.
Indeed, according to Theorem~\ref{thm:compose} the exponentiated optimal $Q$-function is a linear combination of the exponentiated library functions.  
Therefore, the weights can be used to modulate the relative importance of the library functions---modelling the situation in which an agent has multiple concurrent objectives of unequal importance.

We illustrate the effect of the weight vector $\mathbf{w}$ using soft $Q$-learning with a temperature parameter $\tau = 1$.
We construct a new task by composing the tasks \texttt{PurpleCircle} and \texttt{BeigeSquare}, and assign different weights to these tasks. 
The different weighted value functions are given in Figure~\ref{fig:weighted_or}. 

%\newpage

\subsection{--AND-- Composition}

Here we consider tasks which can be described as the intersection of tasks in the library.
In general, this form of composition will not yield an optimal policy for the composite task owing to the presence of local optima in the composed value function.
However, in many cases we can obtain a good approximation to the composite task by simply averaging the Q-values for the constituent tasks.  
While \citet{haarnoja18} considers this type of composition in the entropy-regularised case, we posit that this can be extended to the standard RL setting by taking the low-temperature limit.
We illustrate this by composing the optimal policies for the \texttt{Blue} and \texttt{Square} tasks, which produces a good approximation to the optimal policy for collecting the blue square.  
Results are shown in Figure~\ref{fig:and}.

\subsection{Temporal}

Our final experiment demonstrates the use of composition to long-lived agents. 
We compose the base $Q$-functions for the tasks \texttt{Blue}, \texttt{Beige} and \texttt{Purple}, and use the resulting  $Q$-function to solve the task of collecting \textit{all} objects. Sample trajectories are illustrated by Figure~\ref{fig:temporal}.

Despite the fact that the individual tasks terminate after collecting the required object, if we allow the episode to continue, the composed $Q$-function is able to collect all objects in a greedy fashion.
The above shows the power of composition---if we possess a library of skills learned from previous tasks, we can compose them to solve any task in their union continually. 
%Consider a long-lived agent acting in any environment requiring the collection of resources, such as food, water and firewood. 
%Composing policies for collecting any of these individual resources allows us to collect \textit{all} of them in a best-first manner, without the need for further learning. 
%Additionally, we can assign weights to the resources to create a policy that prioritises one type of resource over another (e.g. preferring water over firewood). 

\begin{figure*}[t!]
	\centering
   	\begin{subfigure}[m]{0.33\textwidth}
   		\centering
    	\includegraphics[height=3.0cm]{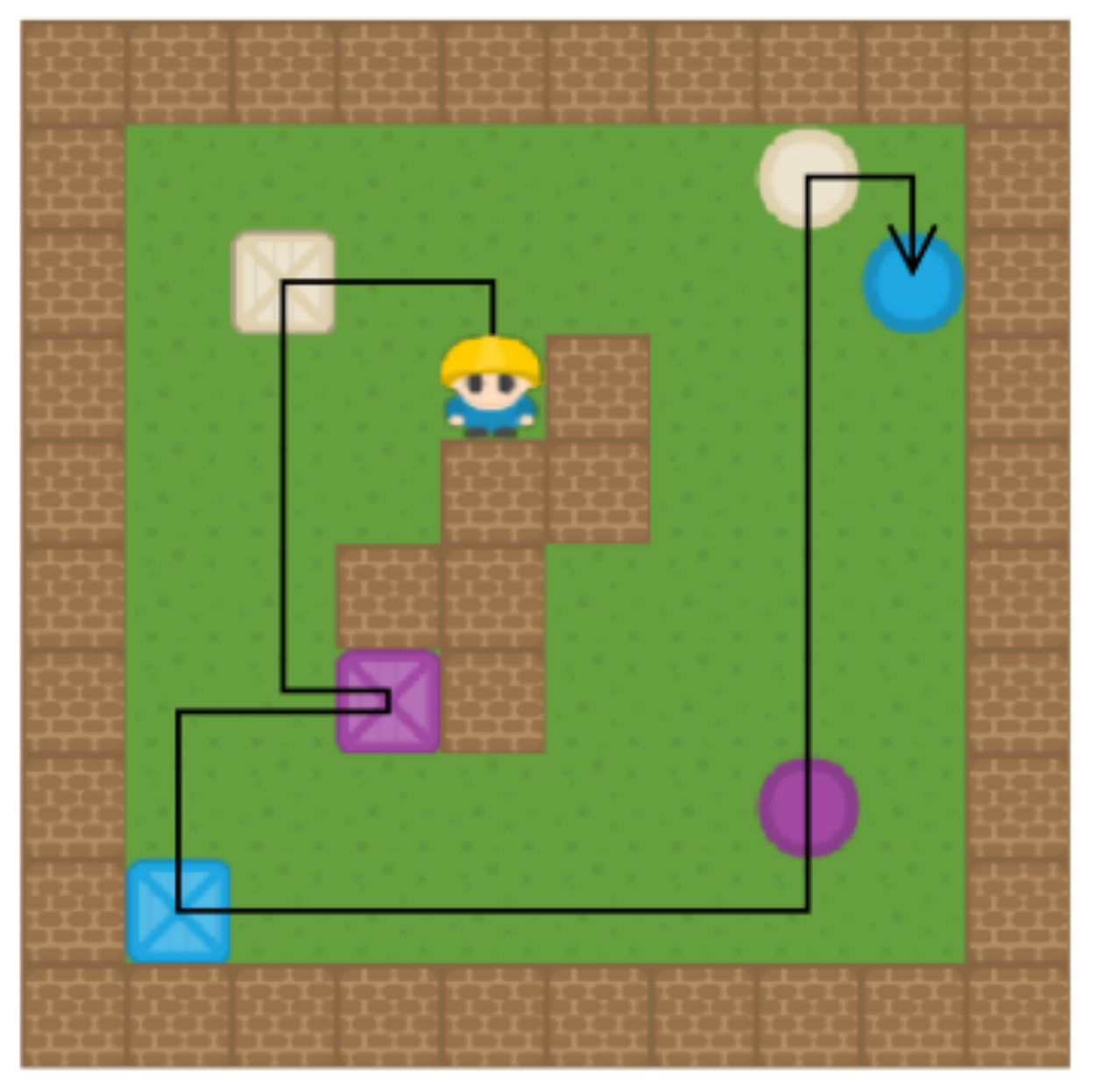}
    	\caption{} \label{fig:trajectories_temporal_1}
	\end{subfigure}%
	~~	
	\begin{subfigure}[m]{0.33\textwidth}
		\centering
      	\includegraphics[height=3.0cm]{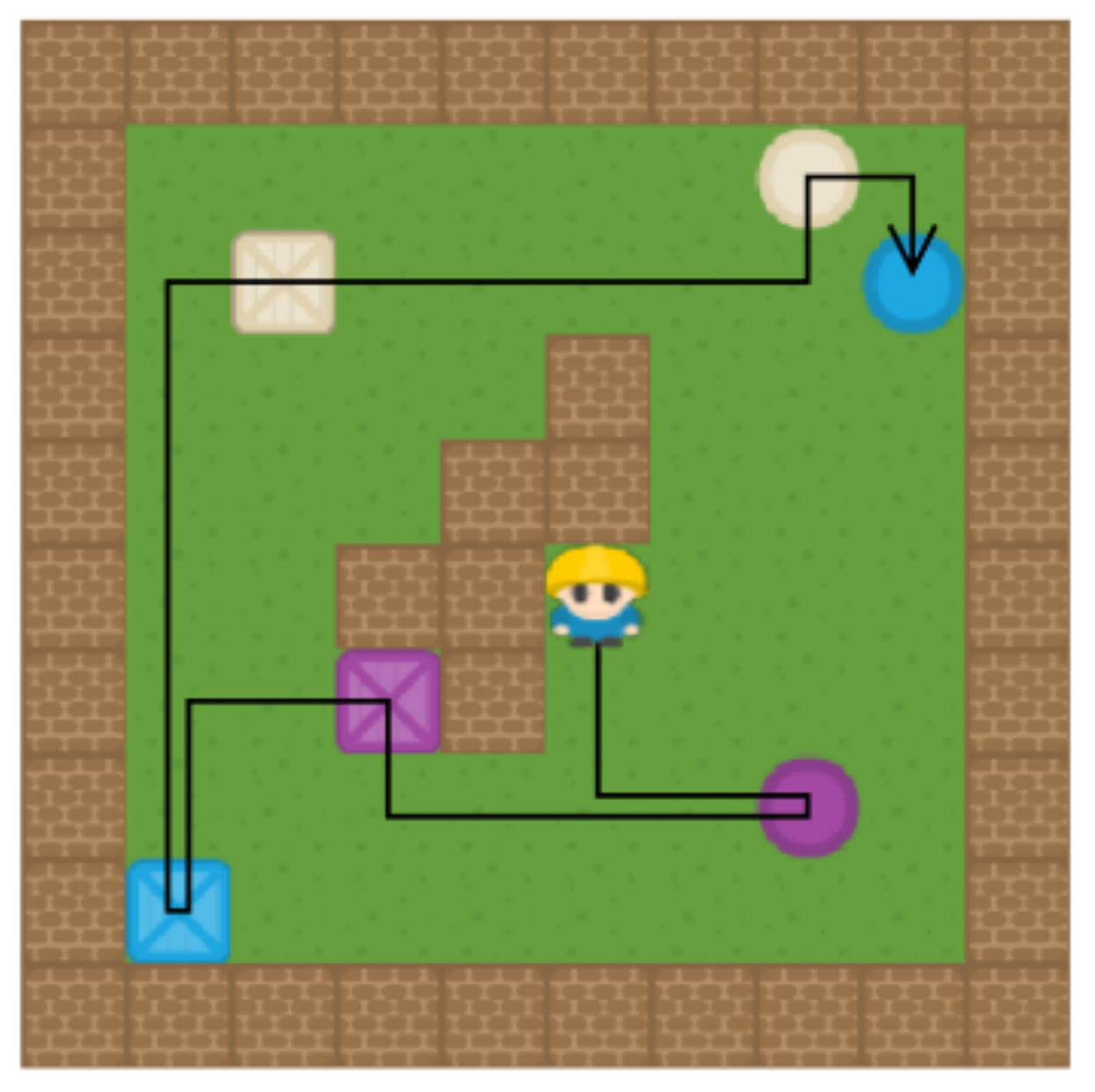}
      	\caption{} \label{fig:trajectories_temporal_2}
	\end{subfigure}%
	\begin{subfigure}[m]{0.33\textwidth}
		\centering
      	\includegraphics[height=3.4cm]{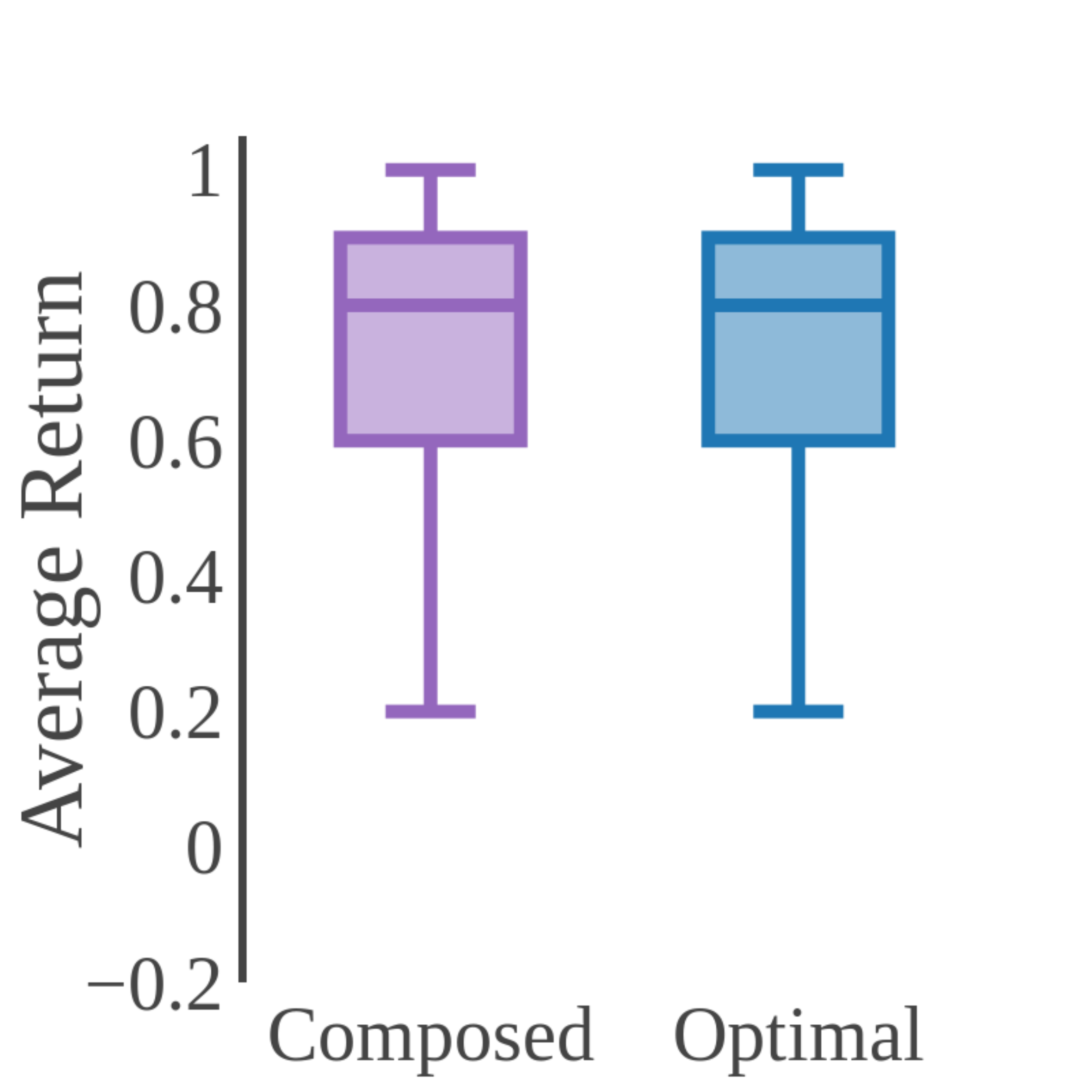}
  	    \caption{} \label{fig:rewards_temporal}
	\end{subfigure}%
	\caption{(\subref{fig:trajectories_temporal_1}) and (\subref{fig:trajectories_temporal_2}) Sample trajectories for the task of collecting all objects. (\subref{fig:rewards_temporal}) Returns from $50$k episodes. The first box plot is the return of the composed $Q$-function, while the second is the result of DQN trained to collect all objects explicitly.} \label{fig:temporal}
\end{figure*}

\section{Conclusion}

We showed that in entropy-regularised RL, value functions can be optimally composed to solve the union of tasks. 
Extending this result by taking the low-temperature limit, we showed that composition is also possible in standard RL.
However, there is a trade-off between our ability to smoothly interpolate between tasks, and the stochasticity of the optimal policy. 
%We showed that entropy regularisation allows for various forms of composition, and discussed the optimality guarantees for each.
We demonstrated, in a high-dimensional environment, that a library of optimal $Q$-functions can be composed to solve composite tasks consisting of unions, intersections or temporal sequences of simpler tasks.
The proposed compositional framework is a step towards lifelong-learning agents that are able to combine existing skills to solve new, unseen problems.

\bibliography{llarla18_composition}
\bibliographystyle{icml2018}

\end{document}